\documentclass[review]{elsarticle}
\usepackage{geometry}
\usepackage{array}
\usepackage{float}
\usepackage{bm}
\usepackage{amssymb}
\usepackage{amsmath}
\usepackage[caption=false,font=footnotesize]{subfig}
\usepackage[ruled,vlined]{algorithm2e}
\usepackage{enumerate}
\usepackage{color}
\usepackage{colortbl}
\definecolor{darkgreen}{rgb}{0, 0.5, 0}
\definecolor{red}{rgb}{1, 0, 0}
\usepackage{booktabs, multirow}
\usepackage{mfirstuc}
\usepackage{lettrine}
\usepackage{pxfonts}
\usepackage{epstopdf}
\usepackage{hyperref, url}

\journal{Applied Soft Computing}

\DeclareMathOperator*{\argmin}{argmin}

\newcommand\etal{\textit{et al.}}
\newcommand\ie{\textit{i.e.}}
\newcommand\eg{\textit{e.g.}}
\newcommand\st{\textit{s.t.}}
\newcommand\wrt{\textit{w.r.t.}}
\newcommand\etc{\textit{etc.}}
\newcommand\algoabbr{TEEA}

\newcommand\titletext{A Reference Vector based Many-Objective Evolutionary Algorithm with Feasibility-aware Adaptation}

\newtheorem{theorem}{Theorem}[section]
\newtheorem{definition}{Definition}[section]

\newtheorem{proof}{Proof}[section]

\begin{document}

\begin{frontmatter}
\title{\titletext{}}
\author[add1,add2, add3]{Mingde Zhao}
\author[add1,add4]{Hongwei Ge\corref{mycorrespondingauthor}}
\cortext[mycorrespondingauthor]{Corresponding author}
\ead{hwge@dlut.edu.cn}
\author[add1]{Kai Zhang}
\author[add1,add5]{Yaqing Hou}

\address[add1]{College of Computer Science and Technology, Dalian University of Technology, China}
\address[add2]{Montr\'eal Institute of Learning Algorithms (Mila), Canada}
\address[add3]{School of Computer Science, McGill University, Canada}
\address[add4]{Department of Computer Science and Engineering, Washington University in St. Louis, USA}
\address[add5]{School of Computer Science and Engineering, Nanyang Technological University, Singapore}

\begin{abstract}
The infeasible parts of the objective space in difficult many-objective optimization problems cause trouble for evolutionary algorithms. This paper proposes a reference vector based algorithm which uses two interacting engines to adapt the reference vectors and to evolve the population towards the true Pareto Front (PF) \st{} the reference vectors are always evenly distributed within the current PF to provide appropriate guidance for selection. The current PF is tracked by maintaining an archive of undominated individuals, and adaptation of reference vectors is conducted with the help of another archive that contains layers of reference vectors corresponding to different density. Experimental results show the expected characteristics and competitive performance of the proposed algorithm \algoabbr{}.
\end{abstract}

\end{frontmatter}

\section{Introduction}
Many-objective Optimization Problems (MaOPs) are one of the biggest open-problems in applied soft computing. The complexities of the real-world problems give rise to the class of heuristic algorithms with population features, which are often recognized as Many-Objective Evolutionary Algorithms (MaOEAs) \cite{louafi2017multi, li2017quantum, ferreira2017design, du2018robust}.
\par
The difficulty of MaOPs increase dramatically with the increment of the dimensionality of the objective space, \ie{} the number of objectives  \cite{purshouse2007pareto, wang2017diversity}. The deterioration in performance inspires new heuristics. Recent literature reveals the trend of combining different strategies to achieve synergistic performance, since sticking to one direction alone leads to the overfitting on problems of certain types and less robustness.
\par
In this paper, we propose a reference vector based algorithm, yet hybridizing the ideas from Pareto dominance relations. The proposed algorithm \algoabbr{} focuses on the scenario of the objective space being haunted by infeasible parts that change the distribution of the PF while it is evolving, which is particularly meaningful for difficult problems where proximity of population to the true PF should not be expected. \algoabbr{} is with a selection engine and an adaptation engine, capable of adapting to the various characteristics of the Feasible Objective Space (FOS) in-time, which can be seen as a generalization of the adaptive reference vector based algorithms aiming to adapt the reference vectors for the true PF only.
\par
The rest of this paper is organized as follows: Section \ref{section:preliminaries} gives the preliminaries of researches on the MaOPs including the basics of reference vector based MaOPs, the focus and goal of this paper, the related works and then provides the ideas that yield the contribution of this work. Section \ref{section:framework} gives the details of the proposed algorithm \algoabbr{}. Section \ref{section:experiments} presents the experimental studies, which include the investigations of the characteristics and the comparative results with other state-of-the-art algorithms on a standard benchmark suite.

\section{Preliminaries}\label{section:preliminaries}
Reference vector based approaches are one of the most popular families of MaOEAs for their many attractive properties on problems with modest number of objectives (\eg{} $M \leq 8$): applicability onto the complicated scenarios \cite{chugh2018surrogate}, human preference articulation \cite{cheng2016reference} and high efficiency, \etc{}. In this section, we will first focus on giving the basic knowledge about the problem that we want to address in this paper, then discuss the related works, as well as our ideas of motivation and inspiration.
\subsection{Basics \& Challenge of Infeasibility}
The goal for MaOEAs is to find a solution set whose objective value vectors in the objective space constitute a ``good'' representative set of the true PF. To be a good representative set, the finite individuals\footnote{we overload the term ``individual'' to call one solution in the solution space and the corresponding objective vector in the objective space.} in the set should well present the spatial properties of the infinite true PF hypersurface. With such recognition, researchers summarize the goal of MaOEAs into obtaining proximity and diversity simultaneously, where proximity ensures that the individuals are close to the true PF and diversity ensures that the limited individuals can well represent the distribution of the true PF.
\par
Without the loss of generality, assuming that the goal of each objective is minimization and the objectives are strictly positive, here we introduce some basic knowledge about reference vector based MaOEAs, which is to be used in the later sections.
\begin{definition}[Feasibility]
Given a many-objective function $\bm{f}: \mathcal{S} \to \mathbb{R}^M_{+}$, where $\mathcal{S} \subset \mathbb{R}^D$ is the solution space defined for the problem (domain of $f$), $D$ is the dimensionality of search space and $M$ is the dimensionality of the objective space, the \textbf{feasible objective space} is defined as
$$\mathcal{O} \equiv \{\bm{o} | \bm{o} = \bm{f}(\bm{x}), \forall \bm{x} \in \mathcal{S}\} \subset \mathbb{R}^M_{+}$$
Also, an objective vector $\bm{o}$ is said to be \textbf{infeasible} if $\bm{o} \notin \mathcal{O}$.
\end{definition}

\begin{definition}[Pareto-dominate]
Given an MaOP and two individuals $\bm{a}, \bm{b} \in \mathbb{R}^M_{+}$, $\bm{a}$ is said to \textbf{Pareto-dominate} (\textbf{dominate}) $\bm{b}$ if and only if $\bm{a}$ is less than or equal to $\bm{b}$ element-wise and with at least one element strictly less than that of $\bm{b}$, \ie{}
$$\bm{a} < \bm{b}\text{ \textbf{\textit{iff}} }\forall i \in \{1, \dots, M\}, a_i \leq b_i\text{ and }\exists j \in \{1, \dots, M\}, a_j < b_j$$
\end{definition}

\begin{definition}[Pareto Front]
Given the many-objective function $\bm{f}: \mathcal{S} \to \mathcal{O}$ to be minimized and the set $P_+$ of individuals ever found in the optimization process, the \textbf{Pareto Front} (\textbf{PF}, \textbf{current PF}) $PF(P_+)$ is a subset of $P_+$ that dominates all other elements in $P_+$, \ie{}
$$\forall \bm{o} \in PF(P_+)\text{ and }\forall \bm{p} \in P_+ - PF(P_+), \bm{o} < \bm{p}$$
\end{definition}

People also overload the term ``current PF'' for the imaginary hypersurface that the current PF locates on.

\begin{definition}[True Pareto Front]
Given the many-objective function $\bm{f}: \mathcal{S} \to \mathcal{O}$ to be minimized, the \textbf{true Pareto Front} (\textbf{true PF}) $truePF(\bm{f})$ is a subset of $\mathcal{O}$ that dominates all other elements in $\mathcal{O}$, \ie{}
$$truePF(\bm{f}) \subset \mathcal{O}, \forall \bm{o}_{PF} \in truePF(\bm{f})\text{ and }\forall \bm{o} \in \mathcal{O} - truePF(\bm{f}), \bm{o}_{PF} < \bm{o}$$
\end{definition}

\begin{definition}[Infeasibility of FOS, Full FOS \& Partial FOS]
Given MaOP $\bm{f}: \mathcal{S} \to \mathcal{O}$ and $truePF(\bm{f})$, the \textbf{trivial infeasible part} $I_{trivial}$ of the objective space is defined as the set of points in $\mathbb{R}^M_+$ that could dominate at least one point in the true PF, \ie{}
$$I_{trivial} \equiv {\bm{o} | \exists \bm{o}_{PF} \in truePF(\bm{f}), \bm{o} < \bm{o}_{PF}}$$
Additionally, if $\mathbb{R}^M = I_{trivial} + \mathcal{O}$, the problem $\bm{f}$ is said to have a \textbf{full feasible objective space} (\textbf{full FOS}); Else, the problem is said to have non-trivial infeasible parts in the objective space or simply a \textbf{partial feasible objective space} (\textbf{partial FOS}).
\end{definition}

Problems with full FOS are also recognized as the simplex-like problems \cite{cheng2016reference, tian2017indicator}. A full FOS occupies the space of $\mathbb{R}^M_{+}$ except for the trivial infeasible part where points dominate the true PF. A partial FOS has more infeasible parts where points are dominated by the true PF.

\begin{definition}[Reference Vectors \& Reference Points]
Given the many-objective function $\bm{f}: \mathcal{S} \to \mathcal{O}$, a set $Z$ of \textbf{reference vectors} is a finite subset of $\mathbb{R}^M_{+}$.
\par
For any reference vector $\bm{z}$, the unit $L_1$-norm vector $\bm{\hat{z}} \equiv \bm{z} / \|\bm{z}\|_1$ is said to be the corresponding \textbf{reference point}.
\end{definition}

Reference vectors and reference points have injective correspondence. We call the scaled reference vectors reference points for they locate on the subspace of the unit hyperplane with $L_1$-norm $1$.

\begin{definition}[Activation of Reference Vector, Association \& Adjacent Space]
Given a set of reference vectors $Z \subset \mathbb{R}^M$ and an individual $\bm{o} in \mathcal{O}$, a reference vector $\bm{z} \in Z$ is said to be \textbf{activated} by $\bm{o}$ if
$$\angle{}(\bm{o}, \bm{z}) = \argmin_{\bm{z}^{'} \in Z}{\angle{}(\bm{o}, \bm{z}^{'})}$$
where $\angle{}(\bm{o}, \bm{z}) \equiv arccos(\frac{\bm{o}^T \bm{z}}{\|\bm{o}\|_2 \|\bm{z}\|_2})$ is the angle between $\bm{o}$ and $\bm{z}$.
Also, $\bm{o}$ is said to be attached or \textbf{associated} to $\bm{z}$ if $\bm{z}$ is activated by $\bm{o}$. The set of all the possible points that can activate $\bm{z}$ given $Z$ is recognized as the \textbf{adjacent space} of $\bm{z}$ given $Z$.
\end{definition}

If a reference vector is activated, there must be individuals nearby the direction the reference vector points towards. An MaOP with full FOS is expected to activate all reference vectors if they are roughly evenly distributed towards every direction in $\mathbb{R}^M_{+}$.

\begin{definition}[Infeasibility of PF, Full PF \& Partial PF]
Given MaOP $\bm{f}: \mathcal{S} \to \mathcal{O}$ and $truePF(\bm{f})$, the MaOP is said to have a \textbf{fully-covering PF} (or \textbf{full PF}) if
$$\forall \bm{o} \in \mathbb{R}^M_{+}, \exists \bm{o}_{PF} \in truePF(\bm{f})\text{ and }c \in \mathbb{R}, \bm{o} = c \bm{o}_{PF}$$
Else, the problem is said to have a \textbf{partially-covering PF} (or \textbf{partial PF}) .
\end{definition}

\begin{theorem}
MaOPs with partial true PFs must have partial FOS.
\end{theorem}

\begin{proof}
Partial PFs cannot dominate all points in a FOS with only the trivial infeasible part. Thus the FOS must have a non-trivial infeasible part.
\end{proof}

An MaOP with full FOS must have full PF, but the converse is not true.
\par
Reference vectors are often used as the guidelines for evolutionary selection. The individuals associated to a reference vector will be selected using some criterion that penalizes the deviation from the reference vector's direction for diversity and simultaneously awards the individuals with shorter projection lengths on the reference vector's direction for proximity, given the reference vectors are correctly set. However, for problems with partial FOS, fixed reference vectors are no longer good guidelines, as indicated in Fig. \ref{fig:guidelines}.
\begin{figure*}[!t]
\centering
\subfloat[Ideal guidelines for selection]{
\captionsetup{justification = centering}
\includegraphics[width=0.4\textwidth]{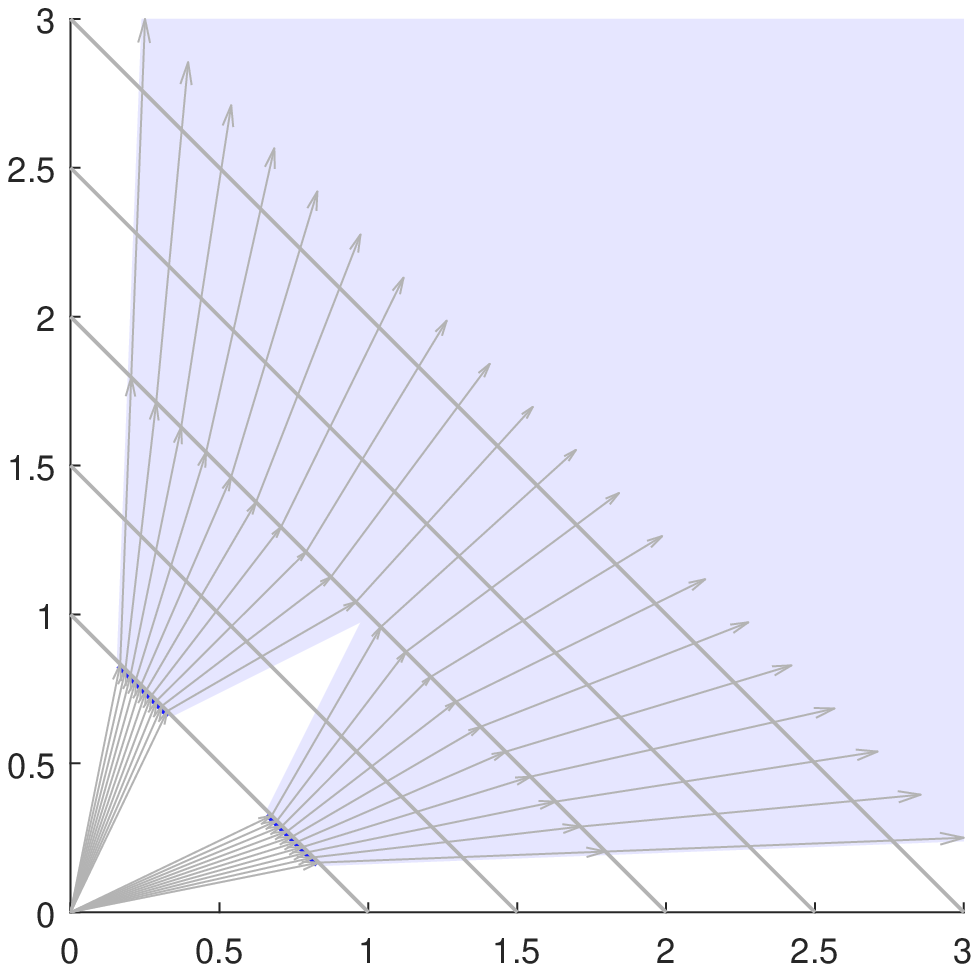}}
\hfill
\subfloat[Reference vectors should have approximated the guidelines]{
\captionsetup{justification = centering}
\includegraphics[width=0.4\textwidth]{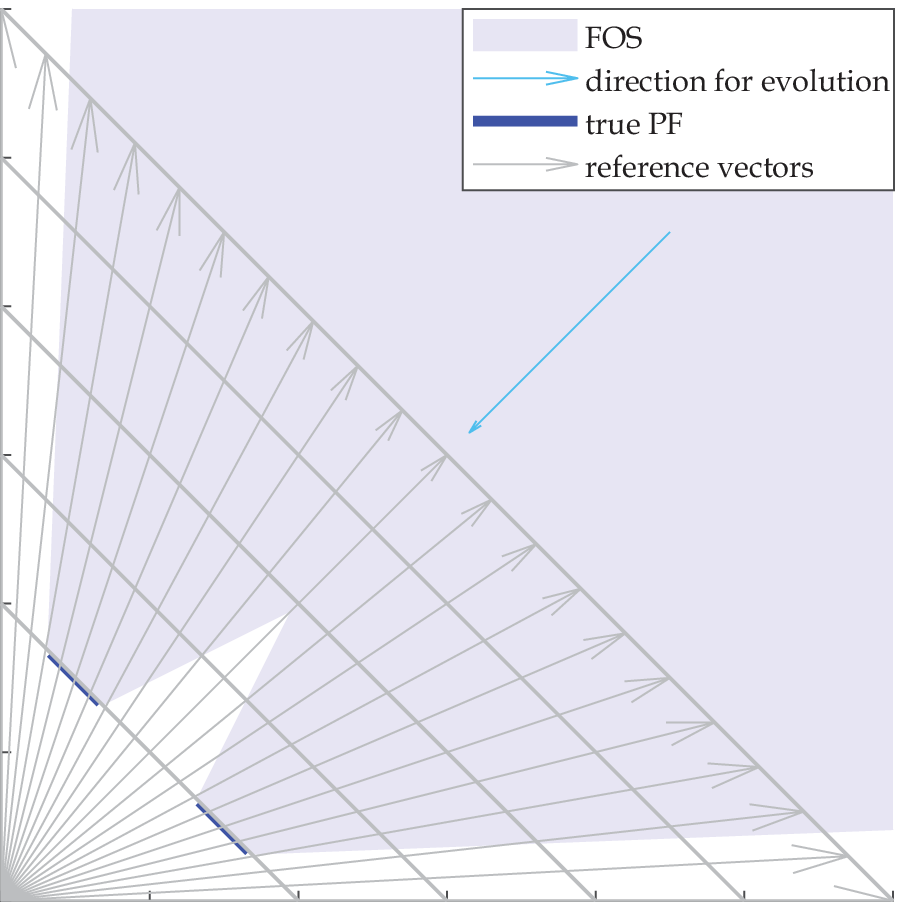}}

\caption{Reality for reference vectors being the guidelines for selection. The FOS is shaded in light blue and the true PF is outlined in dark blue, with the reference vectors drawn as grey quivers. In (a), the best possible guidelines that segment the FOS evenly is given using prior knowledge about the problem; In (b), reference vectors are uniformly generated using the traditional methods, which have huge space left for adaptation.}
\label{fig:guidelines}
\end{figure*}
\par
Reference vector based algorithms with a set of fixed reference vectors are designed for full FOS problems. Previous works show that fixed reference vectors, if not set with prior knowledge, are misleading for the selection of individuals toward diversity for problems with partial true PFs \cite{zhao2018fast, ge2018interacting}. But what do we do if the distribution of the current PF is changed due to the encountered non-trivial infeasible parts of the objective space? Can we adapt the reference vector in-time \st{} they approximate the ideal guidelines at the intersection points of the guidelines with the current PF?

\subsection{Related Works \& Motivations}
There are a lot of existing works for adjusting the reference vectors with the goal of handling irregular PFs.
\par
Some works focus on dealing with the curvature of true PF, MOEA/D-AWA deletes the overcrowded reference vectors and inserting new reference vectors in the sparser spaces \cite{qi2014adaptive}; RVEA in \cite{cheng2016reference} tilts the unit hyperplane to cater to the distribution of the individuals to address the problems caused by the different magnitudes of objectives; \cite{wu2018learning} proposes to use Gaussian process to fit the hypersurface of the current PF for the distribution of reference vectors.
\par
There are also works focusing on the infeasibility of the true PF, \ie{} partial true PFs \cite{ge2018interacting}: Approaches like A-NSGA-III \cite{jain2014evolutionary} and RVEA* \cite{cheng2016reference} employ the strategy of inserting extra reference lines into crowded areas with the risk of the disturbance of uniformity of the distribution of reference vectors. To deal with the disturbance of uniformity, Zhao and Ge \etal{} proposed a series of reference vector algorithms with interactive components capable of learning the complex effective areas (scaled projection of the true PF onto the unit hyperplane with $L_1$-norm 1) if the problems are easy enough to converge early.
\par
These works all share one underlying assumption: the MaOP is easy enough for selection engines to obtain population at the vicinity of the true PF. But for the problems with infeasible parts in the objective space and high difficulty, \eg{} problems with large dimensionality of the solution space and partial FOS, such assumption will not hold. The inadaptability of the existing algorithms on these problem motivate us to address such problem. In 2013, Wang \etal{} contributed a prototype algorithm PICEA-w \cite{wang2013preference} which co-evolves the set of reference vectors alongside the population for a similar purpose, though not tested and still have the problems for the disturbance of uniformity and so on. Can we extend the adaptation of reference vectors to a in-time manner \st{} they might appropriately serve as the right guidelines of selection for the evolving population? The answer is yes - we found a viable solution in the interacting component framework.

\subsection{Inspiration}
Here we propose a general idea to the in-time adaptability to the population in the FOS: if we are able to track the current PF of the evolution reasonably precisely, we can use the tracked PF to somehow adapt the reference vectors, in which case we can approximate the ideal selection guidelines, indicated in Fig. \ref{fig:guidelines}, using the reference vectors!
\par
The problems left for implementing such idea are mainly two:
\begin{enumerate}
\item
How to track the current PF in an effective and efficient way?
\item
How to adjust the references using the tracked PF?
\end{enumerate}
\par
Our solution to the first problem, \ie{} tracking the current PF, is not a parametric model since such model may accidently do interpolation and extrapolation within the infeasible spaces. Thus, we use something simpler: since individuals found during the evolution process never violate the feasibility of the FOS, we just buffer them in an archive! Such idea has been exploited for similar purposes in \cite{Praditwong2006archive, wang2015improved}.
\par
The second problem is way more tricker. Our solution, the main novelty of this paper, is a delicate mechanism constituted of two subroutines that propagate the distribution of the current PFs onto layers of reference vectors corresponding to different densities. Such mechanism can effectively adapt to shrinking or expanding cross-sections of the current PF and the FOS without with marginal suffering of the disturbance of uniformity.
\par
In this paper, we provide an algorithm \algoabbr{} that contains two assisting archives and two interacting engines, aiming to provide versatility for feasibility change of the FOS, as demonstrated in Fig. \ref{fig:interaction}.

\section{\algoabbr{}}
\label{section:framework}
\begin{figure}
\centering
\includegraphics[width = 0.55\textwidth]{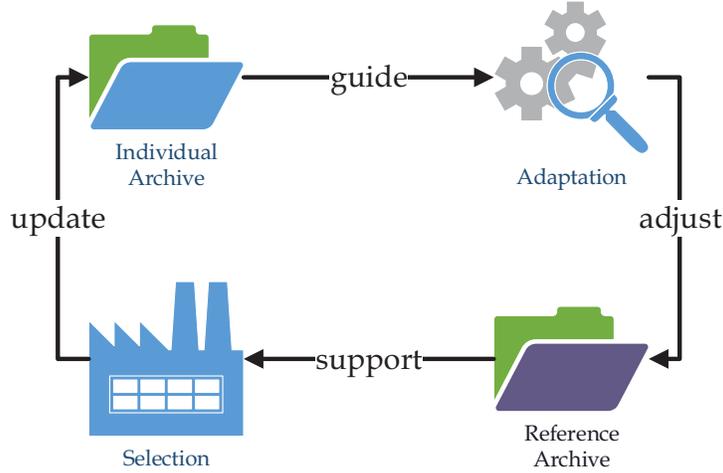}\\
\caption{Interaction between the two engines with the participation of the two archives}
\label{fig:interaction}
\end{figure}
\subsection{Individual Archive}
The current PF is constituted by the objective vectors of the individuals that have never been dominated by the individuals ever found in the optimization process. The information of the current PF is necessary to identify and adapt to the changes in the FOS. However, since typical selection methods for population only leaves $N$ individuals, parts of the information about the current PF will be potentially lost. The individuals that have not been dominated should be archived and maintained as another population (apart from the ``population'' we use to generate offsprings) to keep such information, but not in an ever-increasing manner: we should archive a reasonable amount of individuals that have not been dominated which could well reflect the spatial distribution of the current PF. This calls a proper elimination method for the most undesirable archived individuals. Thus we seek to employ cascade clustering \cite{zhao2018fast, ge2018interacting}, a population selection strategy that keeps the distribution of the population as complete as possible, for the update of such archive, which we named the Individual Archive (IA). Whenever new individuals are found, we use cascade clustering to find and archive the cluster centers (the best individual for the adjacent spaces of reference vectors \wrt{} the criterion of cascade clustering) among the population constituted of newly evolved individuals and the archived individuals. Different to employing cascade clustering for population selection, since IA maintenance only seeks to track the distribution of PF, we only have to buffer the cluster centers that reflect the distribution of PF and discard the rest. This means that IA keeps at least $|Z|$ individuals, where $|Z|$ is the number of participating reference vectors.
\par
Assuming the number of reference vectors is a moderate function of $N$, \ie{} $|Z| = \mathbb{O}(N)$, the maintenance of IA is running at the complexity level $\mathcal{O}(MN^2)$, which depends on the complexity of cascade clustering.
\subsection{RA: Reference Archive}
The distribution of the reference vectors that participate the selection should change in-time to cater with the shape of the current PF. We propose another archive for the reference vectors, which we named the Reference Archive (RA). Since the shape of current PF changes, the local densities of reference vectors should also change. Thus we have designed RA to be hierarchical: each layer inside RA only contains reference vectors generated using the same density; layers with different densities can be combined to achieve higher density. The details of updating the RA will be introduced with the adaptation engine of the reference vectors.

\begin{algorithm}[!t]
\caption{construct $i$-th layer of RA: \textbf{new\_layer($A_R$, $i$)}}
\label{code:updateHRA}
\KwIn{$A_{R}$ (the current RA), $i$}
\KwOut{$\langle Z, \bm{a}, \bm{b}\rangle$ (tuple of reference vectors $Z$, association vector $\bm{a}$, status vector $\bm{b}$)}

\textcolor{darkgreen}{//get the current density of reference points in RA}\\
$d = \text{get\_density}(A_{R}.l_{i - 1})$;\;

\textcolor{darkgreen}{//generate a set of reference vectors with higher density}\\
$Z = \text{uniform\_points}(d)$;\;

\textcolor{darkgreen}{//eliminate redundant reference vectors}\\
$Z^{'} = \emptyset$
\For{$l_i \in A_{R}$}{
    $Z^{'} = Z^{'} \cup A_{R}.l_i.Z$;

}
$Z = Z - Z^{'}$;\textcolor{darkgreen}{//set difference operation}\\

\textcolor{darkgreen}{//boolean vector representing if the corresponding reference vector is enabled, all initialized as disabled}\\
$\bm{b} = \bm{0}_{|Z| \times 1}$;\\

\textcolor{darkgreen}{//index vector associating the reference vectors in the new layer with the ones archived in RA. ``associateto($P$, $Q$)'' associates $\forall \bm{p} \in P$ with the nearest $\forall \bm{q} \in Q$}\\
$\bm{a}$ = associateto($Z$, $Z^{'}$);\\
\end{algorithm}

\subsection{Selection Engine}
This paper provides no novelty in neither the generation of offsprings nor selection of population. We employ the selection method of cascade clustering proposed and ameliorated in \cite{zhao2018fast, ge2018interacting}. Briefly speaking, cascade clustering selects a population that is evenly spread \wrt{} the given reference vectors in a cascade style to achieve proximity and diversity. It has shown state-of-the-art effectiveness and efficiency when compared to state-of-the-art reference vector based selection methods. Also, it has flexible interfaces for the adaptations of reference vectors, which suits the need of this paper. We used this to maintain the IA, now we also use this to select the population. Using cascade clustering both to maintain IA and to select population, we can make sure that the reference vectors that could be activated will always keep at least one associated individual, without the fear of losing them in the population selection. Surprisingly, we can also prove that using cascade clustering to do two tasks sequentially is equivalent to doing them two at the same time. The combined selection and update of the IA is demonstrated in Fig. \ref{fig:cascade_clustering}.
\begin{figure}
\centering
\includegraphics[width = 0.65\textwidth]{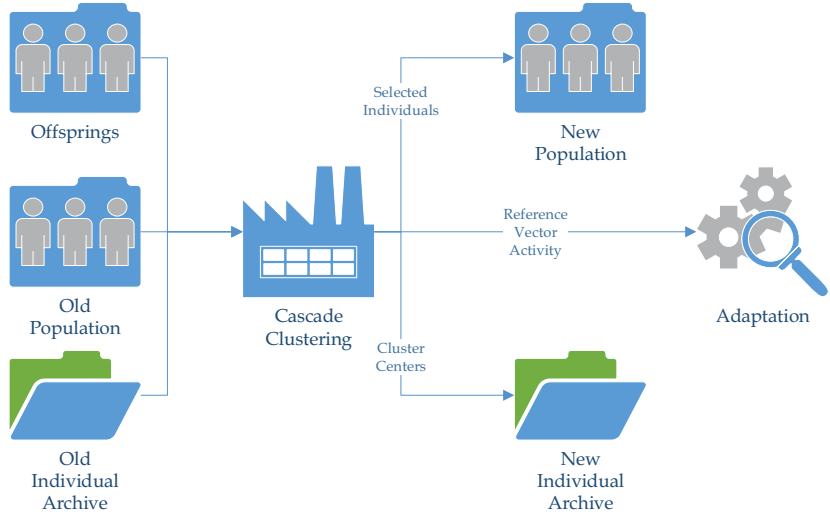}\\
\caption{Combining two cascade clustering processes into one: the equivalent process}
\label{fig:cascade_clustering}
\end{figure}
\par
The overall runtime complexity for the process is $O(MN^2)$, assuming $|Z| = \mathcal{O}(N)$. We provide the pseudo code of the selection based on cascade clustering in Alg. \ref{alg:cc}, without diving into its details since we made no amelioration. For more details, refer to \cite{ge2018interacting}.
\begin{algorithm}[!t]
\small
\caption{Cascade Clustering: \textbf{$P$, $I$, $Q$ = CC($P$, $Z$, $N$)}}
\label{alg:cc}
\KwIn{$Z$ (set of reference vectors), $P$ (potential population), $N$ (population size for the next generation)}
\KwOut{$P$ (population for the next generation), $I$ (indices of active reference vectors), $Q$ (cluster centers)}

\textcolor{darkgreen}{//frontier individual identification}\\

$[F, NF] \leftarrow $ frontier\_identification($P$)\;

%

\textcolor{darkgreen}{//attach frontiers to reference vectors, return the clusters}\\
$C \leftarrow $attach($F$, $Z$, 'point2vector')\;
$I \leftarrow $cluster2indices($C$)\;

\For{each cluster $c_i \in C$}{
    \For{each frontier $\bm{f_j}$ in $c_i$}{
        $PDM(\bm{f_j}) \leftarrow mean(\bm{f_j}) + sin(\bm{z_i}, \bm{f_j})$;
    }
    $c_i.F \leftarrow $ sort($c_i.F$, $PDM(c_i.F)$, ascend)\;
    Pick out ${c}_i.{f}_j$ with the smallest PDM as $c_i.center$;
}

\textcolor{darkgreen}{//attach non-frontiers to clusters}\\

$C  \leftarrow $attach($NF$, $C$, 'point2center')\;

\For{each cluster $c_i \in C$}{
    $c_i.NF \leftarrow$ sort($c_i.NF$, $d(c_i.NF, c_i.center)$, ascend)\;
    create selection queue $c_i.S \leftarrow \langle c_i.F, c_i.NF \rangle$\;
}
\textcolor{darkgreen}{//round-robin picking}\\
$i \leftarrow 1$\;
$P \leftarrow \emptyset$\;
\While{$|P| < N$}{
    $P \leftarrow P \cup Pop(c_i.S)$\;
    $i \leftarrow mod(i, |C|) + 1$;
}

\textcolor{darkgreen}{//specify the cluster centers}\\
$Q \leftarrow \emptyset$;
\For{each cluster $c_i \in C$}{
    $Q \leftarrow Q \cup \{c_i.center\}$;
}
\end{algorithm}
\par
Cascade clustering additionally returns the indices of the active reference vectors, which correspond to the reference vectors attached with nondominated individuals, telling the adaptation where the reference vector should be \ie{} the distribution of the current PF.
\subsection{Adaptation Engine}
Reference vectors serve as guidelines for the selection of individuals towards proximity and diversity in the objective space. For problems with full FOS, simple reference vectors generated could serve as the ideal guidelines. However, for objective spaces with unfeasible parts, the ideal guidelines become much more complicated and unadjusted reference vectors are far from ideal guidelines, as demonstrated in Fig. \ref{fig:guidelines} (a). For these objective spaces, sticking to the fixed reference vectors hurts diversity and potentially proximity as well. It is needed to adjust the reference vectors in-time \st{} they are always good approximations to vectors towards the intersection points of the guidelines and the current PF, as illustrated in Fig. \ref{fig:shrink}.
\par
One possible solution to such approximation for guidelines of arbitrarily but continuously shrinking and expanding FOS is evenly distributing sufficient number of reference vectors that intersect with the current PF\footnote{The conjecture is metaphoric: the current PF is assumed to be a hypersurface but actually just a set of points in the objective space; Intersection points of reference vectors with the current PF means the orthogonal projection of the objective vectors of the individuals onto the associated reference vectors.}. We can easily come up with a brute force method: buffer sets of evenly distributed reference vectors with different generation densities. After each update of the IA, starting from the set of reference vectors corresponding to the lowest density\footnote{Any generation density corresponding to generating less than $N$ reference vectors should be excluded from the possible selections of generation density.}, see how many reference vectors intersect with the current PF and continue to do so until the number of active reference vectors is around $N$.
\par
The problem with this brute force method is that it is too costly: each adaptation costs approximately $\mathcal{O}(MN^3)$. We seek to contribute a smooth and gradual approximation to such brute force method with significantly less computational cost for adaptations. The idea is simple: for the moment of adjustments, if the number of intersection points of reference vectors and the current PF is below $N$, perhaps due to the shrinkage or insufficient covering of the of current PF, we initialize the ``shrink'' subrountine: generate a set of reference points with a higher density and save this set in the top layer of RA but only enable those within the adjacent spaces of all currently active reference points (disabled points will not participate in the selection). Then we get all the enabled reference vectors stored in every layer of RA to as the participating set of reference vectors. The set addition of evenly distributed reference points yields approximately evenly distributed reference vectors that will potentially intersect the current PF; If the number of intersection points is greater than $N$, which is perhaps due to the expansion of the distribution of current PF, we initialize the ``expand'' subrountine: enable the reference points in the lower layers that are within the adjacent spaces of the active reference points in the densest layer and then delete the densest layer\footnote{In implementation, simply disable it}.
\par
Precalculating the layers of reference vectors and the angles between them reduces the complexity of this process to $\mathcal{O}(MN^2)$, which is exactly what we aim for: significantly lower than the brute force method of $\mathcal{O}(MN^3)$.
\begin{figure}
\centering
\subfloat[fixed reference vectors]{
\captionsetup{justification = centering}
\includegraphics[width=0.28\textwidth]{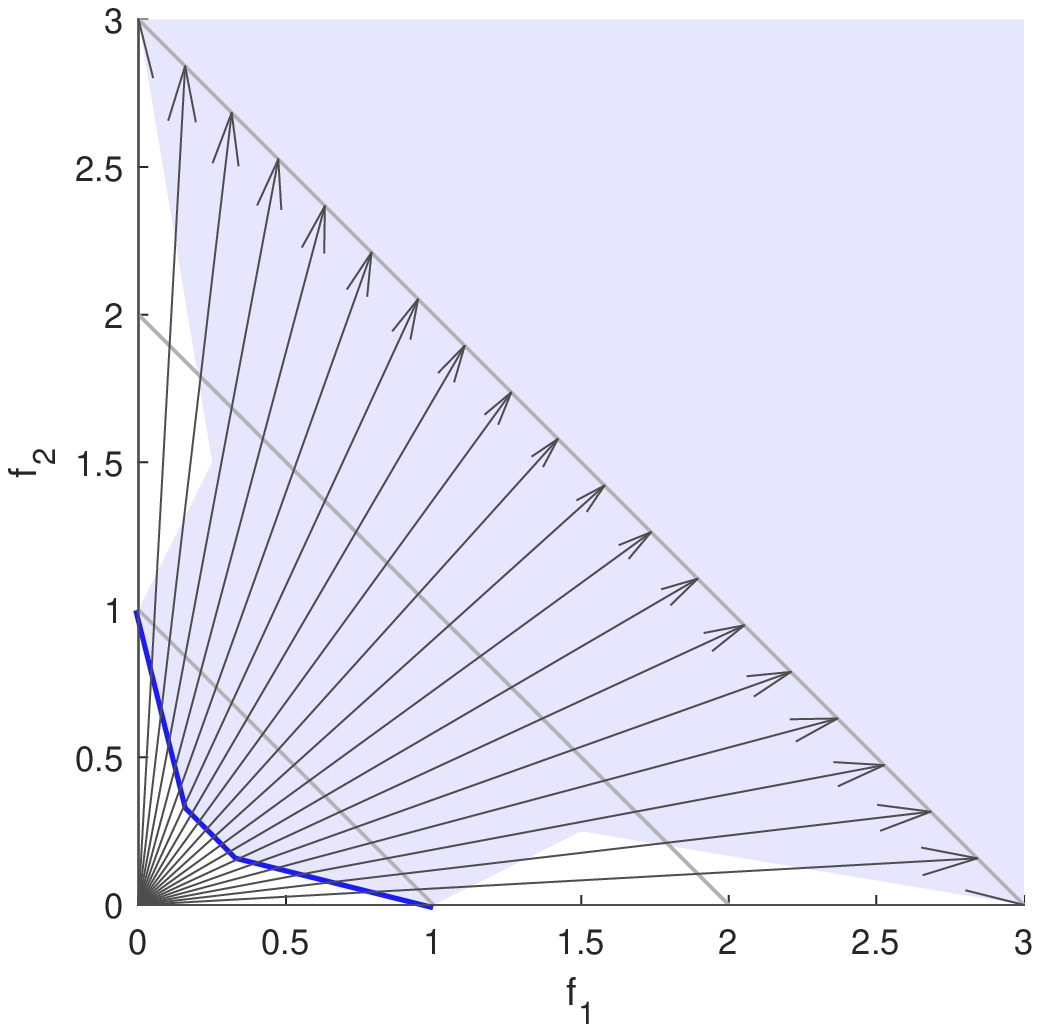}}
\hfill
\subfloat[before RA adjustment]{
\captionsetup{justification = centering}
\includegraphics[width=0.28\textwidth]{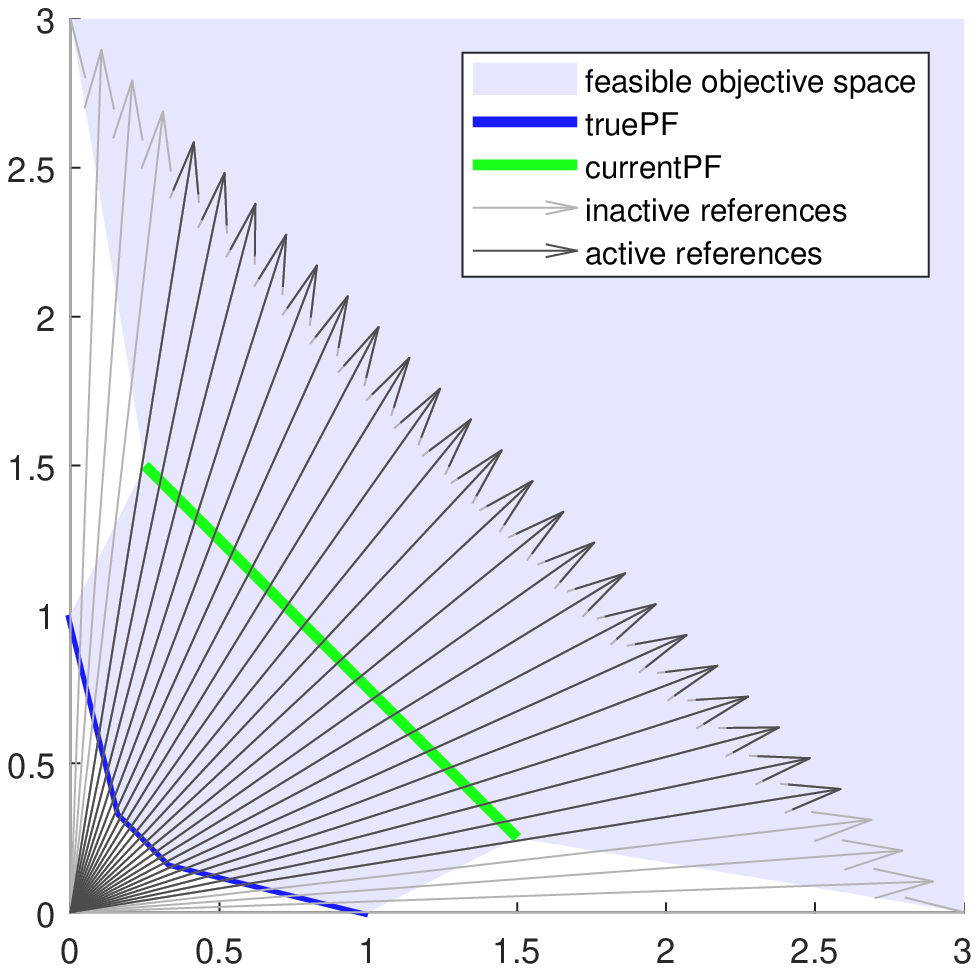}}
\hfill
\subfloat[after RA adjustment]{
\captionsetup{justification = centering}
\includegraphics[width=0.28\textwidth]{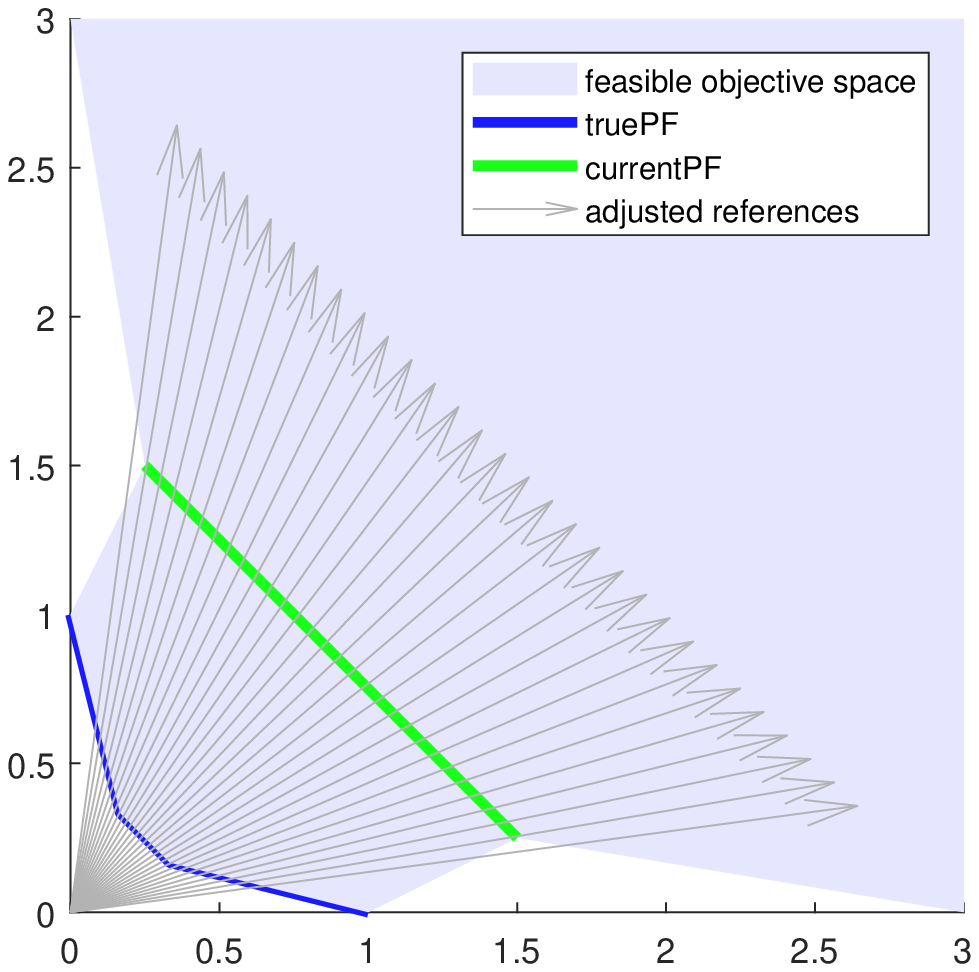}}

\caption{The demonstration for the shrinking of the current PF due to the characteristics of the FOS. The FOS is colored in light blue and the PF is outlined in dark blue. The reference vectors located at the edge of the FOS has inactive parts. Thus the distribution is not always appropriate.}
\label{fig:shrink}
\end{figure}

In the RA, the reference vectors will be uniformly increased and uniformly reduced, thus the disturbance of the uniformity of the reference vectors is relatively low. Visualization of simulated initializations of adaptation is presented in Fig. \ref{fig:adaptation}, with the pseudocode given in Alg. \ref{alg:adaptation}.

\begin{algorithm}[!t]
\caption{Adaptation}
\label{alg:adaptation}
\KwIn{$A_{R}$ (RA), $A_{I}$ (IA), $Z_{active}$ (active reference vectors), $N$ (global population size), $\theta$ (tolerance ratio)}
\KwOut{$A_{R}$ (updated RA), $Z$ (participating reference vectors)}

\textcolor{darkgreen}{//subroutine ``SHRINK''}\\
\If{$|Z_{active}| < (1 - \theta)N$}{

    \textcolor{darkgreen}{//generate a new layer, retrive the corresponding $Z, `\bm{a}', `\bm{b}'$}\\
    $Z, \bm{a}, \bm{b}$ = new\_layer($A_{R}$);\textcolor{darkgreen}{//algorithm \ref{code:updateHRA}}\\
    $A_{R} = A_{R} \cup \{\langle Z, \bm{a}, \bm{b} \rangle\}$;\\

    \textcolor{darkgreen}{//enable the newly generated reference vectors if they are associated with currently active ones in RA}\\
    \For{$\bm{z}_i \in Z$}{
        \If{$A_{R}[a_i] \in Z_{active}$}{ \textcolor{darkgreen}{//the reference vector which $\bm{z}_i$ is associated to is in $Z_{active}$}\\
            $b_i = 1$;\textcolor{darkgreen}{//enable}\\
        }
    }
}

\textcolor{darkgreen}{//subroutine ``EXPAND''}\\

\If{$|Z_{active}| > (1 + \theta)N$}{

    \textcolor{darkgreen}{//extract all layers of RA, implementation only requires indices}\\

    $Z_{|A_{R}|}, \bm{a}_{|A_{R}|}, \bm{b}_{|A_{R}|}$ = depack($A_{R}.l_{|A_{R}|}$);\\
    $Z_{|A_{R}| - 1}, \bm{a}_{|A_{R}| - 1}, \bm{b}_{|A_{R}| - 1}$ = depack($A_{R}.l_{|A_{R}| - 1}$);\\
    $\cdots$;\\

    \textcolor{darkgreen}{//back-propagate the distribution of the current PF towards the lower layers}\\
    \For{$i \in {1, \dots, |A_{R}|}$}{
        $\bm{a}_{|A_{R}|-i+1}^{'}$ = associateto($Z_{|A_{R}|-i+1}$, $Z_{|A_{R}|}$);\textcolor{darkgreen}{//back-associate the points in the lower layer to the last layer, \st{} we can use these associations to enable points in the lower layers}\\
        \For{$\bm{z}^i_{|A_{R}|-i+1} \in Z_{|A_{R}|-i+1}$}{
            \If{$A_{R}[a_i^{'}] \in Z_{active} \cap Z_{|A_{R}|}$}{
                $b_{|A_{R}|-i+1}^i = 1$;\textcolor{darkgreen}{//enable if the point in the lower layer is back-associated to the active points in the last layer}
            }
        }
    }
    $A_{R} = A_{R} - A_{R}.l_{|A_{R}|}$; \textcolor{darkgreen}{//symbolic, disabling is wiser}
}

\textcolor{darkgreen}{//accumulate enabled reference vectors as the set of participating reference vectors}\\
$Z = \emptyset$;\\

\For{$l_i \in A_{R}$}{
    $Z_i, \bm{a}, \bm{b}$ = depack($l$);\\
    \For{$\bm{z}_j \in Z_i$}{
        \If{$b_j = 1$}{
            $Z = Z \cup \{\bm{z}_j\}$;
        }
    }
}
\end{algorithm}

\begin{figure*}
\centering
\subfloat[Initialization of the current PF (population) and the reference vectors]{
\captionsetup{justification = centering}
\includegraphics[width=0.32\textwidth]{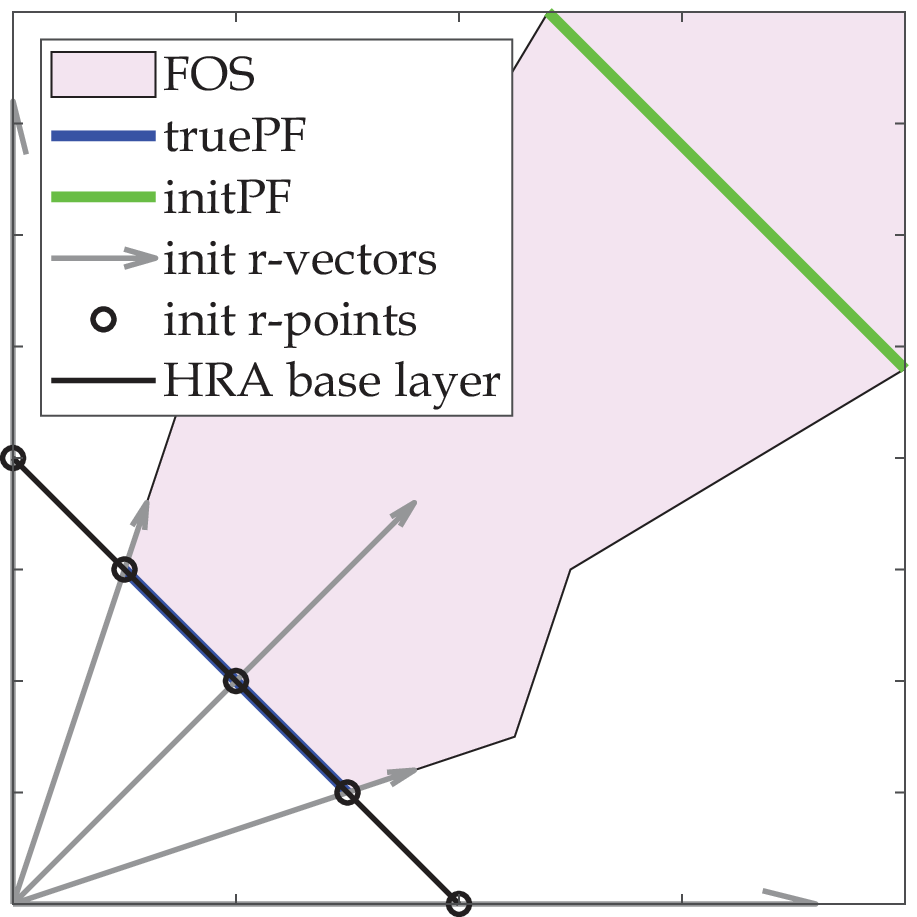}}
\hfill
\subfloat[To shrink: insufficient active reference vectors]{
\captionsetup{justification = centering}
\includegraphics[width=0.32\textwidth]{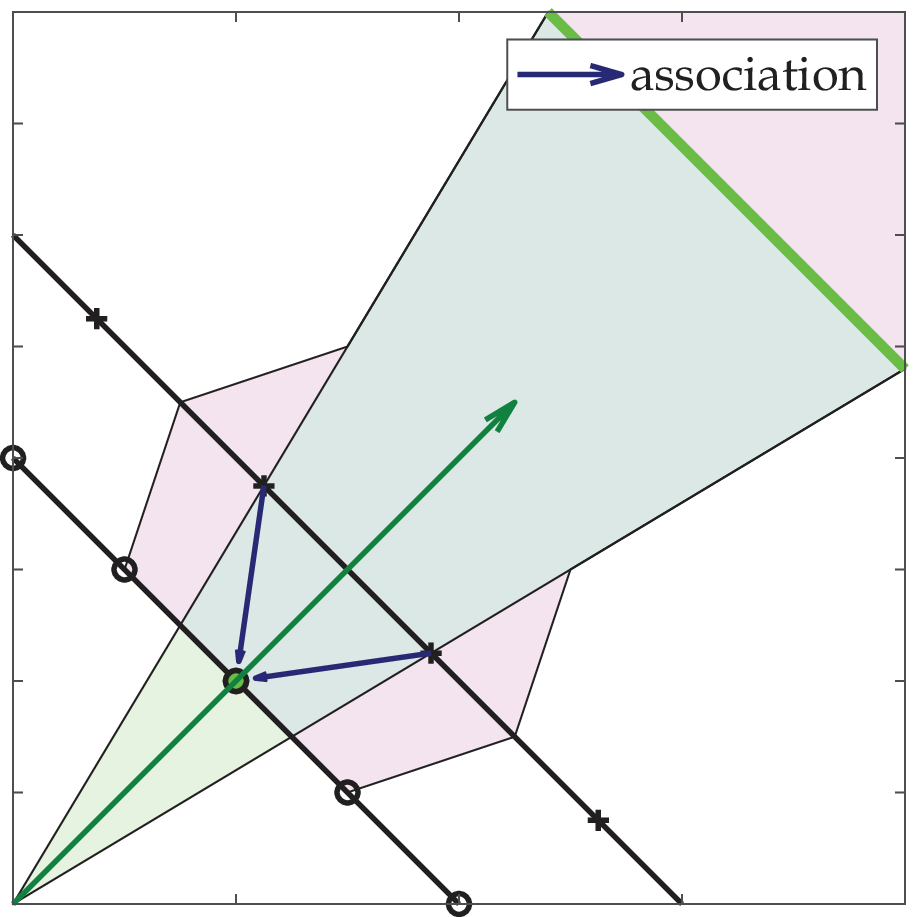}}
\hfill
\subfloat[Shrinked: enabled points on the second layer associating to previously active points. Active reference vectors still insufficient]{
\captionsetup{justification = centering}
\includegraphics[width=0.32\textwidth]{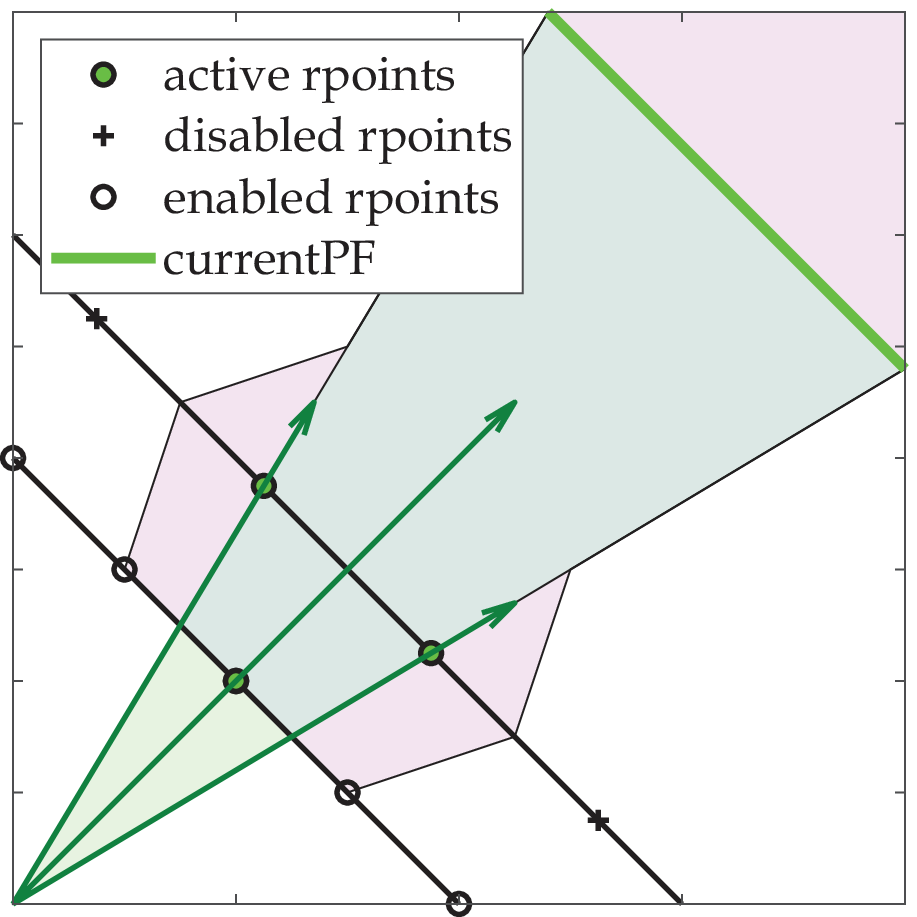}}

\subfloat[Second shrink: appropriate number of active reference vectors]{
\captionsetup{justification = centering}
\includegraphics[width=0.32\textwidth]{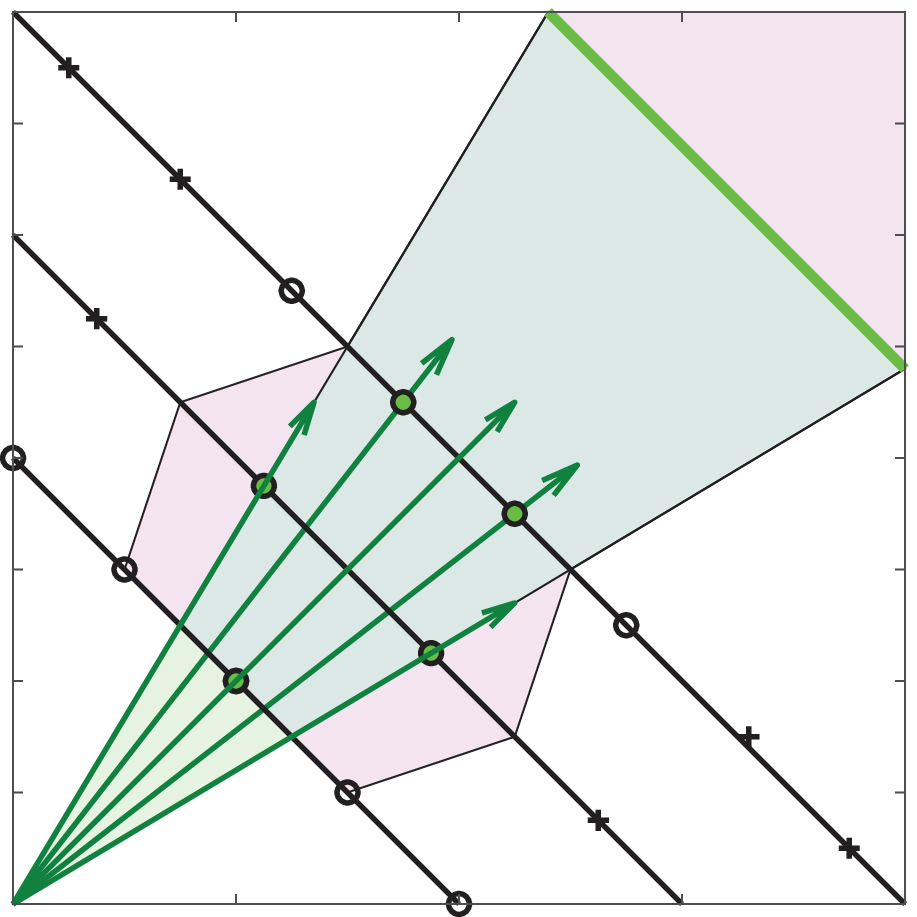}}
\hfill
\subfloat[To expand: too many active reference vectors]{
\captionsetup{justification = centering}
\includegraphics[width=0.32\textwidth]{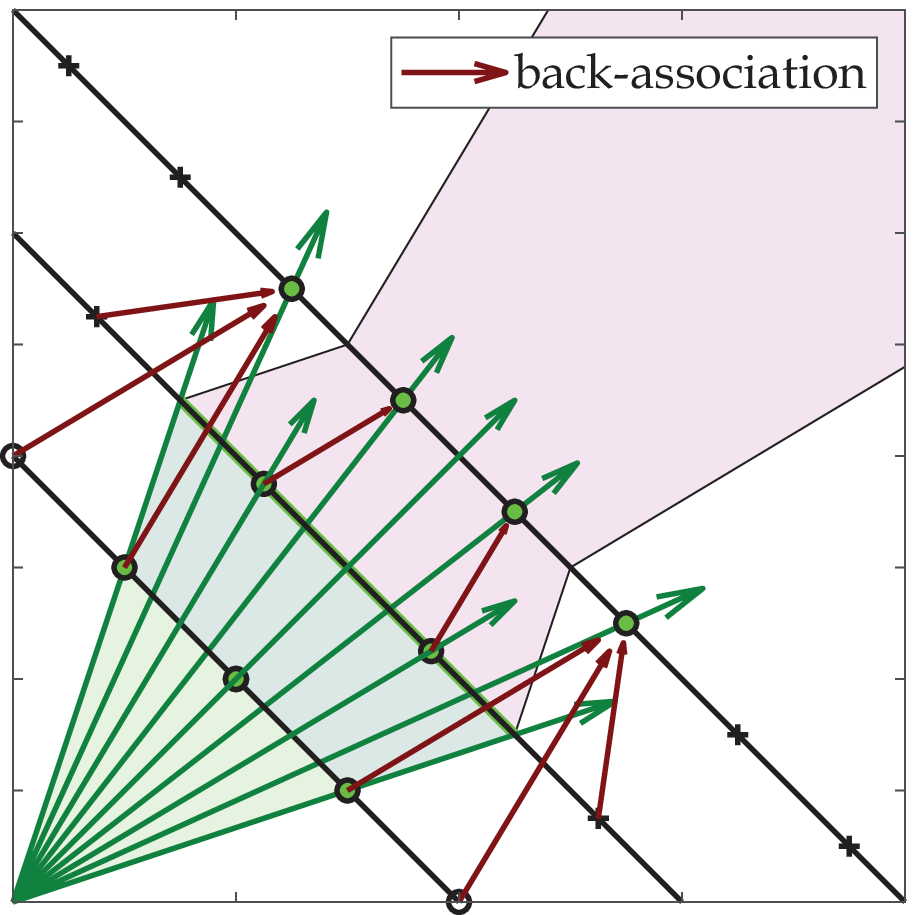}}
\hfill
\subfloat[Expanded: appropriate number of active reference vectors]{
\captionsetup{justification = centering}
\includegraphics[width=0.32\textwidth]{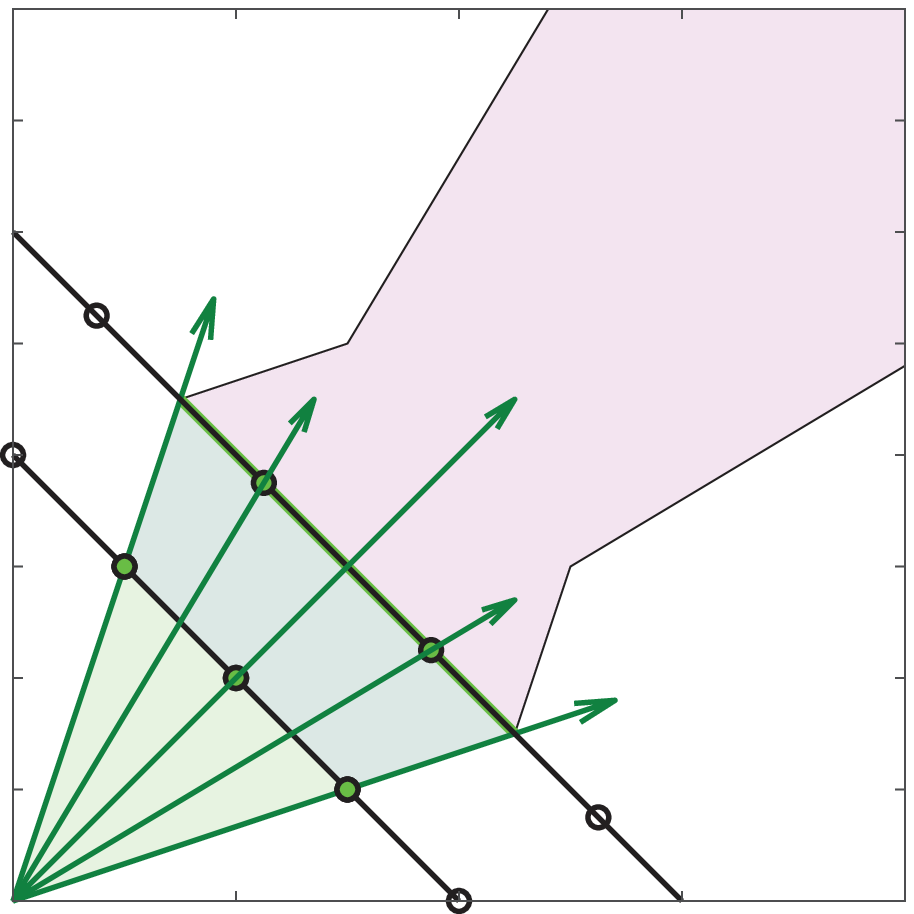}}

\caption{Demonstration of the subroutines ``shrink'' and ``expand'' on an imaginary problem with non-trivial infeasible parts in the FOS: (a). We assume that the population is initialized far away from the true PF, constituting a hyperplane that is roughly parallel to the unit hyperplane. Also, we assume that the currentPF evolves gradually towards the true PF, representing a scenario in which optimization difficulties on the objectives are isotropic; (b). The number of active reference vectors are not enough, thus the subroutine ``shrink'' is activated: the reference vectors in second layer of RA that are associated with the currently activated reference vectors are enabled; Note that the reason why the reference points in the layers are non-overlapping is that the reference vectors that can be found in previous layers are excluded; (c). After the first shrink, there are more active reference vectors, but still not enough. Need more shrinks; (d). After the second shrink, the number of active reference vectors are appropriate, which means it is roughly equivalent to $N = 5$ in the demonstration; (e). The distribution of current PF expands. The number of active reference vectors is too high and thus the subroutine ``expand'' will be triggered: the distribution of the active reference vectors (the distribution of the central projection of the current PF onto the unit hyperplane) will be back-propagated from the densest layer to the previous layers via the computed back-associations; (f). After the expansion, there are again appropriate number of reference vectors intersecting the current PF roughly evenly. Notice that the forward- and back-propagation both enables points located on the edge of the central projection \st{} the change of the current PF can be detected using the number of activated reference vectors.}
\label{fig:adaptation}
\end{figure*}
\subsection{Combining All Together}
We now give the proposed algorithm \algoabbr{} which combines the $4$ proposed components. At the beginning, to initialize, $N$ individual solutions are randomly generated and $A_R$ is initialized with $N$ reference points uniformly generated on the unit hyperplane as the base layer. Then, the main cycle loops:
\begin{enumerate}
    \item Evolve the old population to get the $N$ offsprings by employing a certain heuristic;
    \item Pick $N$-sized population and new individual archive out of the union of the last population, the newly evolved offsprings and the individuals in the individual archive using cascade clustering. Additionally, provide the activity of reference vectors.
    \item Check the stability of the activity of the reference vectors, report whether it is appropriate to adjust the distribution of the reference vectors in the RA according to IA. If not \textbf{continue};
    \item Initialize adaptation, which may trigger one of the two subroutines. The set of participating reference vectors will be provided to the selection engine for selection.
\end{enumerate}
When the termination criteria are satisfied, we pick final population using the individuals buffered in the IA using cascade clustering. Note that in order to make \algoabbr{} more stable, \ie{} reduce the number of unnecessary adaptation, we introduce a mechanism of checking the stability of activity: if the activity of each reference vector does not change for $w$ generations, the mechanism reports stable else unstable; Else, the adaptation will not be triggered.
\section{Experimental Studies}\label{section:experiments}
This section gives the analyses for the effectiveness of \algoabbr{} on benchmark problems, including the validation of the proposed components, comparison with the state-of-the-art algorithms, \etc{}.
\subsection{Settings}
The settings of the experiments are identical to those for the CEC'2018 MaOP competition \cite{cheng2017benchmark}, which uses the MaF benchmark suite, where $15$ scalable many-objective benchmark functions. Among them, $8$ cases are with infeasible parts in the objective spaces (MaF1, MaF2, MaF4, MaF6 - MaF9, MaF15), $7$ cases are with full FOS (MaF3, MaF5, MaF10 - MaF14). Specifically, MaF8 and MaF9 are with narrow FOS, MaF6 is with extremely narrow (degenerate) FOS, MaF3 and MaF13 are hard to converge, MaF7 is with disconnected FOS and PF, MaF14 and MaF15 are with large-scale solution space. Each benchmark function is scaled to three separate test cases, in which $M = 5$, $M = 10$ or $M = 15$. There is an injective mapping between $M$ and $D$ for each function, and the number of maximum evaluations for each algorithm is set to be $D(M) \times 10 ^ {4}$. All algorithms are restrict to using the population size of $N = 240$, as well as the classical simulated binary crossover and polynomial mutation optimizer.
\par
Results given in this section are averaged over $20$ independent runs on identical platforms with MATLAB R2019a. The source code\footnote{The MATLAB source code will be published in Mingde Zhao's github: \url{https://github.com/PwnerHarry/}} of \algoabbr{} is implemented with PlatEMO 2.0 \cite{PlatEMO}.
\par
We evaluate the quality of the obtained populations using the Inverse Generational Distance based on $L_2$ norm ($L_2$-IGD, often abbreviated as IGD), a problem dependent evaluation criterion of both proximity and diversity \cite{czyzzak1998pareto}. IGD calculates the average minimum distance from the sample points on the true PF to the points of the population. The smaller the IGD, the better the proximity and diversity.

\subsection{Hyperparameters}
Here we want to determine a fixed set of hyperparameters for all the following experiments, instead of overfitting them for each problem, for fairness of comparison. Also, we want to check if the proposed algorithm \algoabbr{} is sensitive to these hyperparameters. There are two hyperparameters in \algoabbr{}, the window size $w$ to suggest adaptation moments and tolerance ratio $\theta$ for the sensitivity of adaptation \wrt{} the reference vector activities.
\par
Note that when the number of objectives is large, $\theta$ becomes sensitive for the difficulty in generating uniform reference vectors. When the number of evaluations given is low, $w$ becomes sensitive in a problem-specific way for it controls the trade-off between adaptation accuracy (the accuracy of the moments of the adaptation being initialized) and number of adaptations before finish. We select the test case DTLZ7 \cite{deb2005scalable} specifically since it has fractal FOS and intermediate number of evaluations, which is likely an expected scenario for the experiments to come.
\par
For the sake of fair comparison, the set of problems we use for hyperparameter analysis does not have intersection with the experiments later. The results are given in Tab. \ref{tab:hyperparams}. Observing from the results, we can find an approximate interval for the hyperparameters to achieve similarly good performance. We choose $w = 20$ and $\theta = 0.2$ since the change of performance around this point is relatively modest and also for its good performance.
\begin{table*}[htbp]
\setlength{\tabcolsep}{3pt}
\scriptsize
  \centering
  \caption{Hyperparameter Pairs on DTLZ7 ($M = 5$)}
    \begin{tabular}{cccccc}
    \toprule
    \toprule
    \textbf{$w$\textbackslash{}$\theta$} & 0.05 & 0.1 & 0.15 & 0.2 & 0.25 \\
    \midrule
    10    & \cellcolor[rgb]{ .973,  .412,  .42}6.63e-1 & \cellcolor[rgb]{ .988,  .682,  .475}4.85e-1 & \cellcolor[rgb]{ 1,  .894,  .514}3.44e-1 & \cellcolor[rgb]{ .667,  .824,  .498}3.00e-1 & \cellcolor[rgb]{ 1,  .922,  .518}3.25e-1 \\
    15    & \cellcolor[rgb]{ .976,  .431,  .424}6.52e-1 & \cellcolor[rgb]{ .992,  .714,  .478}4.65e-1 & \cellcolor[rgb]{ .698,  .831,  .498}3.02e-1 & \cellcolor[rgb]{ .58,  .8,  .49}2.93e-1 & \cellcolor[rgb]{ .922,  .898,  .51}3.19e-1 \\
    20    & \cellcolor[rgb]{ .988,  .698,  .475}4.75e-1 & \cellcolor[rgb]{ .992,  .776,  .49}4.23e-1 & \cellcolor[rgb]{ .412,  .749,  .482}2.80e-1 & \cellcolor[rgb]{ .478,  .769,  .486}2.85e-1 & \cellcolor[rgb]{ .843,  .875,  .506}3.13e-1 \\
    25    & \cellcolor[rgb]{ .906,  .894,  .51}3.18e-1 & \cellcolor[rgb]{ .659,  .824,  .498}2.99e-1 & \cellcolor[rgb]{ .388,  .745,  .482}2.78e-1 & \cellcolor[rgb]{ .388,  .745,  .482}\textbf{2.78e-1} & \cellcolor[rgb]{ .698,  .831,  .498}3.02e-1 \\
    50    & \cellcolor[rgb]{ .984,  .561,  .451}5.66e-1 & \cellcolor[rgb]{ .988,  .682,  .475}4.85e-1 & \cellcolor[rgb]{ .996,  .78,  .49}4.21e-1 & \cellcolor[rgb]{ .992,  .761,  .49}4.33e-1 & \cellcolor[rgb]{ .992,  .714,  .478}4.65e-1 \\
    \bottomrule
    \bottomrule
	\multicolumn{6}{m{0.3\textwidth}}{\tiny Color indicators are added to for assisting the reading of the results. The greener the better performance, the redder the worse.}\\
	\multicolumn{6}{m{0.3\textwidth}}{\tiny Averaged from $20$ runs on DTLZ7 with $M = 5$, $D = 24$ and $\text{FEs} = 2.4e5$}
    \end{tabular}%
  \label{tab:hyperparams}%
\end{table*}%

\subsection{Comparative Tests on MaF Benchmark Suite}
We try to analyze the characteristics of \algoabbr{} by comparing it with state-of-the-art algorithms including CVEA3 \cite{yuan2018cost}, BCE-IBEA \cite{li2016pareto}, fastCAR \cite{zhao2018fast}, CLIA \cite{ge2018interacting}, AR-MOEA \cite{tian2017indicator}, A-NSGA-III \cite{jain2014evolutionary} and RVEA* \cite{cheng2016reference}. Among them, fastCAR, CLIA, A-NSGA-III and RVEA* are reference vector based with adaptation methods. The characteristics of the compared algorithms are presented in Tab. \ref{tab:algorithms}.
\begin{table*}[htbp]
\setlength{\tabcolsep}{3pt}
\scriptsize
  \centering
  \caption{Details for Compared Algorithms}
    \begin{tabular}{cm{0.8\textwidth}}
    \toprule
    \toprule
    \textbf{Algorithms} & \multicolumn{1}{c}{\textbf{Comments}} \\
    \midrule
    TEEA  & The algorithm proposed in this paper. \\
    CVEA3 & Cost value based MaOEA 3, best-performing (1st) participants of CEC'2018 MaOP competition \cite{yuan2018cost}. The evolution operator has been set to default for fair comparison. \\
    BCE-IBEA & Bi-criterion variant of IBEA \cite{zitzler2004indicator}, one of the most best-performing (3rd) participants of CEC'2018 MaOP competition \cite{li2016pareto}. \\
    fastCAR & Reference vector based MaOEA with periodic adaptation based on margin learning, one of the most best-performing (4th) participants of CEC'2018 MaOP competition \cite{zhao2018fast}. \\
    CLIA  & Improvement upon fastCAR. Reference vector based MaOEA with incremental learning of the PF via component interactions \cite{ge2018interacting}.\\
    AR-MOEA & An indicator based MaOEA based on adaptive reference points \cite{tian2017indicator}.  \\
    A-NSGA-III & NSGA-III with reference vector adaptation \cite{jain2014evolutionary}.  \\
    RVEA* & RVEA with reference vector adaptation \cite{cheng2016reference}.\\
    \bottomrule
    \bottomrule
    \end{tabular}%
  \label{tab:algorithms}%
\end{table*}%

\par
The averaged IGD results of the algorithms are presented in Tab. \ref{tab:IGD}, with additional statistical information including Friedman tests and paired $t$-tests, \etc{}. The Friedman tests suggest that with high confidence we can say that \algoabbr{} the rankings of the algorithms are effective, therefore telling that \algoabbr{} achieves the overall performance. The paired $t$-tests suggest that \algoabbr{} beats several compared state-of-the-art algorithms and achieves similar performance with others. The Friedman tests, $t$-tests as well as straight preservations from the result table all suggest that \algoabbr{} achieves undeniably competitive performance on such set of complex benchmarks problems.
\begin{table*}[htbp]
\setlength{\tabcolsep}{1pt}
\tiny
  \centering
  \caption{Averaged IGD Results on MaF Suite}
    \begin{tabular}{cccccccccccccccccc}
    \toprule
    \toprule
    \multirow{2}[2]{*}{Problem} & \multirow{2}[2]{*}{M} & \multicolumn{2}{c}{\algoabbr{}} & \multicolumn{2}{c}{CVEA3} & \multicolumn{2}{c}{BCE-IBEA} & \multicolumn{2}{c}{fastCAR} & \multicolumn{2}{c}{CLIA} & \multicolumn{2}{c}{AR-MOEA} & \multicolumn{2}{c}{A-NSGA-III} & \multicolumn{2}{c}{RVEA*} \\
          &       & Mean  & Std   & Mean  & Std   & Mean  & Std   & Mean  & Std   & Mean  & Std   & Mean  & Std   & Mean  & Std   & Mean  & Std \\
    \midrule
    \multirow{3}[1]{*}{F1} & 5     & \cellcolor[rgb]{ .396,  .745,  .482}1.08e-1 & 7.25e-4 & \cellcolor[rgb]{ 1,  .918,  .518}1.13e-1 & 1.26e-3 & \cellcolor[rgb]{ .388,  .745,  .482}1.08e-1 & 2.70e-4 & \cellcolor[rgb]{ .949,  .906,  .514}1.12e-1 & 1.20e-2 & \cellcolor[rgb]{ .671,  .824,  .498}1.10e-1 & 7.55e-4 & \cellcolor[rgb]{ .984,  .6,  .459}1.41e-1 & 1.89e-3 & \cellcolor[rgb]{ .973,  .412,  .42}1.58e-1 & 1.45e-2 & \cellcolor[rgb]{ .98,  .545,  .447}1.46e-1 & 8.94e-3 \\
          & 10    & \cellcolor[rgb]{ .388,  .745,  .482}2.33e-1 & 7.45e-3 & \cellcolor[rgb]{ .988,  .918,  .514}2.52e-1 & 4.42e-3 & \cellcolor[rgb]{ .518,  .78,  .486}2.37e-1 & 8.84e-3 & \cellcolor[rgb]{ .988,  .686,  .475}2.96e-1 & 4.67e-7 & \cellcolor[rgb]{ .596,  .804,  .494}2.39e-1 & 3.09e-3 & \cellcolor[rgb]{ 1,  .922,  .518}2.52e-1 & 1.29e-3 & \cellcolor[rgb]{ .996,  .808,  .498}2.73e-1 & 1.11e-2 & \cellcolor[rgb]{ .973,  .412,  .42}3.47e-1 & 2.86e-2 \\
          & 15    & \cellcolor[rgb]{ .439,  .757,  .482}2.67e-1 & 8.85e-3 & \cellcolor[rgb]{ .996,  .831,  .502}3.28e-1 & 4.51e-3 & \cellcolor[rgb]{ .851,  .878,  .506}2.95e-1 & 3.09e-3 & \cellcolor[rgb]{ .996,  .847,  .506}3.24e-1 & 8.07e-12 & \cellcolor[rgb]{ .388,  .745,  .482}2.63e-1 & 2.81e-3 & \cellcolor[rgb]{ .749,  .847,  .502}2.88e-1 & 1.25e-2 & \cellcolor[rgb]{ 1,  .882,  .514}3.15e-1 & 6.29e-3 & \cellcolor[rgb]{ .973,  .412,  .42}4.32e-1 & 3.34e-2 \\
    \multirow{3}[0]{*}{F2} & 5     & \cellcolor[rgb]{ .796,  .863,  .506}9.65e-2 & 2.85e-3 & \cellcolor[rgb]{ .663,  .824,  .498}9.35e-2 & 1.93e-3 & \cellcolor[rgb]{ .388,  .745,  .482}8.72e-2 & 1.16e-3 & \cellcolor[rgb]{ .98,  .557,  .451}1.08e-1 & 1.86e-1 & \cellcolor[rgb]{ .773,  .855,  .502}9.60e-2 & 1.67e-3 & \cellcolor[rgb]{ .973,  .412,  .42}1.11e-1 & 1.03e-3 & \cellcolor[rgb]{ .988,  .678,  .475}1.06e-1 & 1.80e-3 & \cellcolor[rgb]{ .98,  .506,  .439}1.09e-1 & 1.95e-3 \\
          & 10    & \cellcolor[rgb]{ .388,  .745,  .482}1.53e-1 & 2.99e-3 & \cellcolor[rgb]{ .761,  .851,  .502}1.76e-1 & 4.34e-3 & \cellcolor[rgb]{ .714,  .835,  .498}1.73e-1 & 5.04e-3 & \cellcolor[rgb]{ .976,  .451,  .427}2.65e-1 & 2.08e-1 & \cellcolor[rgb]{ .506,  .776,  .486}1.60e-1 & 1.62e-3 & \cellcolor[rgb]{ .996,  .831,  .502}2.06e-1 & 8.73e-3 & \cellcolor[rgb]{ .988,  .663,  .471}2.32e-1 & 2.45e-2 & \cellcolor[rgb]{ .973,  .412,  .42}2.71e-1 & 1.19e-2 \\
          & 15    & \cellcolor[rgb]{ .388,  .745,  .482}1.64e-1 & 4.35e-3 & \cellcolor[rgb]{ .745,  .847,  .502}1.98e-1 & 2.26e-2 & \cellcolor[rgb]{ .973,  .412,  .42}2.86e-1 & 8.86e-3 & \cellcolor[rgb]{ .988,  .647,  .467}2.57e-1 & 1.18e-1 & \cellcolor[rgb]{ .408,  .749,  .482}1.66e-1 & 2.52e-3 & \cellcolor[rgb]{ 1,  .898,  .514}2.26e-1 & 1.02e-2 & \cellcolor[rgb]{ .969,  .91,  .514}2.20e-1 & 1.28e-2 & \cellcolor[rgb]{ .976,  .431,  .424}2.84e-1 & 2.12e-2 \\
    \multirow{3}[0]{*}{F3} & 5     & \cellcolor[rgb]{ .8,  .863,  .506}6.82e-2 & 1.90e-3 & \cellcolor[rgb]{ .388,  .745,  .482}5.93e-2 & 7.58e-4 & \cellcolor[rgb]{ .973,  .412,  .42}1.64e-1 & 4.48e-2 & \cellcolor[rgb]{ .62,  .812,  .494}6.43e-2 & 9.99e-1 & \cellcolor[rgb]{ .584,  .8,  .49}6.36e-2 & 1.12e-3 & \cellcolor[rgb]{ .996,  .784,  .494}9.71e-2 & 1.74e-3 & \cellcolor[rgb]{ 1,  .863,  .51}8.30e-2 & 2.45e-2 & \cellcolor[rgb]{ 1,  .898,  .514}7.68e-2 & 2.65e-3 \\
          & 10    & \cellcolor[rgb]{ .404,  .749,  .482}8.36e-2 & 2.20e-3 & \cellcolor[rgb]{ .561,  .792,  .49}1.69e-1 & 1.19e-2 & \cellcolor[rgb]{ 1,  .922,  .518}6.50e-1 & 3.02e-1 & \cellcolor[rgb]{ .4,  .749,  .482}8.31e-2 & 1.00e0 & \cellcolor[rgb]{ .388,  .745,  .482}7.45e-2 & 2.60e-3 & \cellcolor[rgb]{ 1,  .922,  .518}4.08e0 & 1.18e1 & \cellcolor[rgb]{ .973,  .412,  .42}5.61e6 & 2.01e7 & \cellcolor[rgb]{ 1,  .922,  .518}7.73e0 & 8.66e0 \\
          & 15    & \cellcolor[rgb]{ .412,  .749,  .482}9.11e-2 & 1.27e-3 & \cellcolor[rgb]{ .827,  .871,  .506}1.88e-1 & 3.38e-2 & \cellcolor[rgb]{ 1,  .922,  .518}5.45e-1 & 1.27e-1 & \cellcolor[rgb]{ .404,  .749,  .482}8.97e-2 & 1.00e0 & \cellcolor[rgb]{ .388,  .745,  .482}8.56e-2 & 6.07e-4 & \cellcolor[rgb]{ .996,  .804,  .498}5.48e1 & 1.51e2 & \cellcolor[rgb]{ .973,  .412,  .42}2.35e2 & 6.16e2 & \cellcolor[rgb]{ 1,  .922,  .518}2.66e-1 & 6.95e-1 \\
    \multirow{3}[0]{*}{F4} & 5     & \cellcolor[rgb]{ .502,  .776,  .486}1.77e0 & 3.25e-2 & \cellcolor[rgb]{ .388,  .745,  .482}1.70e0 & 2.24e-2 & \cellcolor[rgb]{ .973,  .412,  .42}3.07e0 & 1.03e0 & \cellcolor[rgb]{ .933,  .902,  .514}2.01e0 & 9.79e-2 & \cellcolor[rgb]{ .733,  .843,  .502}1.90e0 & 6.47e-2 & \cellcolor[rgb]{ .992,  .725,  .482}2.45e0 & 9.25e-2 & \cellcolor[rgb]{ .992,  .776,  .49}2.35e0 & 1.49e-1 & \cellcolor[rgb]{ 1,  .906,  .518}2.09e0 & 9.62e-2 \\
          & 10    & \cellcolor[rgb]{ .8,  .863,  .506}7.50e1 & 2.25e0 & \cellcolor[rgb]{ .388,  .745,  .482}5.10e1 & 2.12e0 & \cellcolor[rgb]{ .988,  .663,  .471}9.34e1 & 2.53e0 & \cellcolor[rgb]{ .855,  .878,  .506}7.81e1 & 3.13e-6 & \cellcolor[rgb]{ .973,  .412,  .42}1.00e2 & 2.36e0 & \cellcolor[rgb]{ .98,  .537,  .447}9.68e1 & 6.42e0 & \cellcolor[rgb]{ .98,  .514,  .439}9.74e1 & 6.05e0 & \cellcolor[rgb]{ .878,  .886,  .51}7.97e1 & 7.62e0 \\
          & 15    & \cellcolor[rgb]{ .91,  .894,  .51}2.80e3 & 1.23e2 & \cellcolor[rgb]{ .388,  .745,  .482}1.49e3 & 1.58e1 & \cellcolor[rgb]{ .882,  .886,  .51}2.74e3 & 1.02e3 & \cellcolor[rgb]{ .996,  .808,  .498}3.24e3 & 1.10e-11 & \cellcolor[rgb]{ .984,  .612,  .459}3.62e3 & 1.62e1 & \cellcolor[rgb]{ .973,  .412,  .42}4.01e3 & 5.26e2 & \cellcolor[rgb]{ .98,  .553,  .447}3.74e3 & 3.39e2 & \cellcolor[rgb]{ .894,  .89,  .51}2.77e3 & 2.46e2 \\
    \multirow{3}[0]{*}{F5} & 5     & \cellcolor[rgb]{ .984,  .914,  .514}1.96e0 & 9.55e-3 & \cellcolor[rgb]{ .98,  .537,  .447}2.28e0 & 1.10e0 & \cellcolor[rgb]{ .388,  .745,  .482}1.75e0 & 3.35e-2 & \cellcolor[rgb]{ 1,  .922,  .518}1.97e0 & 8.12e-1 & \cellcolor[rgb]{ .89,  .89,  .51}1.93e0 & 8.36e-3 & \cellcolor[rgb]{ .973,  .412,  .42}2.39e0 & 8.01e-1 & \cellcolor[rgb]{ .996,  .918,  .514}1.97e0 & 4.90e-3 & \cellcolor[rgb]{ .984,  .569,  .451}2.26e0 & 8.41e-1 \\
          & 10    & \cellcolor[rgb]{ .988,  .69,  .475}8.82e1 & 1.19e0 & \cellcolor[rgb]{ .616,  .808,  .494}6.04e1 & 1.99e1 & \cellcolor[rgb]{ .388,  .745,  .482}4.88e1 & 1.70e0 & \cellcolor[rgb]{ .988,  .69,  .475}8.82e1 & 9.68e-1 & \cellcolor[rgb]{ .973,  .91,  .514}7.80e1 & 1.03e0 & \cellcolor[rgb]{ .973,  .412,  .42}9.87e1 & 4.11e0 & \cellcolor[rgb]{ 1,  .886,  .514}8.08e1 & 4.10e0 & \cellcolor[rgb]{ .816,  .867,  .506}7.03e1 & 1.39e1 \\
          & 15    & \cellcolor[rgb]{ 1,  .914,  .518}2.35e3 & 3.07e2 & \cellcolor[rgb]{ .396,  .745,  .482}1.38e3 & 1.15e3 & \cellcolor[rgb]{ .388,  .745,  .482}1.37e3 & 5.75e1 & \cellcolor[rgb]{ 1,  .871,  .51}2.43e3 & 9.90e-1 & \cellcolor[rgb]{ .969,  .91,  .514}2.29e3 & 2.05e2 & \cellcolor[rgb]{ .973,  .412,  .42}3.31e3 & 4.44e2 & \cellcolor[rgb]{ .984,  .918,  .514}2.32e3 & 6.54e2 & \cellcolor[rgb]{ .996,  .816,  .498}2.54e3 & 7.97e2 \\
    \multirow{3}[0]{*}{F6} & 5     & \cellcolor[rgb]{ .388,  .745,  .482}1.82e-3 & 1.43e-4 & \cellcolor[rgb]{ .541,  .788,  .49}2.20e-3 & 6.87e-5 & \cellcolor[rgb]{ .498,  .776,  .486}2.10e-3 & 7.29e-6 & \cellcolor[rgb]{ .996,  .843,  .506}6.51e-3 & 1.29e-1 & \cellcolor[rgb]{ .624,  .812,  .494}2.40e-3 & 3.32e-4 & \cellcolor[rgb]{ 1,  .902,  .514}4.22e-3 & 5.63e-5 & \cellcolor[rgb]{ .98,  .553,  .447}1.82e-2 & 1.14e-2 & \cellcolor[rgb]{ .973,  .412,  .42}2.37e-2 & 4.98e-3 \\
          & 10    & \cellcolor[rgb]{ .4,  .745,  .482}2.75e-3 & 6.25e-4 & \cellcolor[rgb]{ .388,  .745,  .482}1.99e-3 & 4.35e-5 & \cellcolor[rgb]{ 1,  .851,  .506}4.57e-1 & 3.57e-1 & \cellcolor[rgb]{ .765,  .851,  .502}2.49e-2 & 9.91e-2 & \cellcolor[rgb]{ .918,  .894,  .51}3.40e-2 & 1.28e-4 & \cellcolor[rgb]{ 1,  .91,  .518}1.08e-1 & 1.40e-1 & \cellcolor[rgb]{ .973,  .412,  .42}3.02e0 & 4.22e0 & \cellcolor[rgb]{ 1,  .922,  .518}4.39e-2 & 4.22e-2 \\
          & 15    & \cellcolor[rgb]{ .388,  .745,  .482}1.03e-2 & 9.25e-3 & \cellcolor[rgb]{ .467,  .765,  .486}2.43e-2 & 8.09e-2 & \cellcolor[rgb]{ 1,  .902,  .514}7.46e-1 & 7.25e-3 & \cellcolor[rgb]{ .702,  .835,  .498}6.56e-2 & 9.24e-2 & \cellcolor[rgb]{ .569,  .796,  .49}4.24e-2 & 1.50e-2 & \cellcolor[rgb]{ 1,  .918,  .518}2.38e-1 & 1.25e-1 & \cellcolor[rgb]{ .973,  .412,  .42}1.55e1 & 1.37e1 & \cellcolor[rgb]{ 1,  .922,  .518}1.68e-1 & 1.91e-1 \\
    \multirow{3}[0]{*}{F7} & 5     & \cellcolor[rgb]{ .984,  .914,  .514}2.78e-1 & 3.25e-3 & \cellcolor[rgb]{ .996,  .827,  .502}2.89e-1 & 1.51e-1 & \cellcolor[rgb]{ .529,  .784,  .49}2.27e-1 & 4.63e-3 & \cellcolor[rgb]{ .996,  .804,  .498}2.91e-1 & 2.57e-1 & \cellcolor[rgb]{ .914,  .894,  .51}2.70e-1 & 6.11e-3 & \cellcolor[rgb]{ .973,  .412,  .42}3.29e-1 & 7.55e-3 & \cellcolor[rgb]{ 1,  .906,  .518}2.81e-1 & 1.68e-2 & \cellcolor[rgb]{ .388,  .745,  .482}2.11e-1 & 3.92e-3 \\
          & 10    & \cellcolor[rgb]{ .471,  .769,  .486}8.35e-1 & 1.03e-1 & \cellcolor[rgb]{ .627,  .812,  .494}8.49e-1 & 3.44e-3 & \cellcolor[rgb]{ .388,  .745,  .482}8.28e-1 & 5.58e-2 & \cellcolor[rgb]{ .973,  .412,  .42}1.60e0 & 1.61e-1 & \cellcolor[rgb]{ .675,  .827,  .498}8.53e-1 & 3.69e-2 & \cellcolor[rgb]{ .976,  .435,  .427}1.57e0 & 9.55e-2 & \cellcolor[rgb]{ .992,  .71,  .478}1.18e0 & 9.98e-2 & \cellcolor[rgb]{ 1,  .902,  .518}9.09e-1 & 1.36e-1 \\
          & 15    & \cellcolor[rgb]{ .573,  .796,  .49}1.80e0 & 3.40e-1 & \cellcolor[rgb]{ .388,  .745,  .482}1.64e0 & 4.62e-2 & \cellcolor[rgb]{ 1,  .922,  .518}2.16e0 & 1.77e-1 & \cellcolor[rgb]{ .973,  .412,  .42}1.48e1 & 1.91e-2 & \cellcolor[rgb]{ .996,  .918,  .514}2.16e0 & 1.67e-1 & \cellcolor[rgb]{ .996,  .847,  .506}4.06e0 & 6.97e-1 & \cellcolor[rgb]{ 1,  .882,  .514}3.17e0 & 4.49e-1 & \cellcolor[rgb]{ .537,  .784,  .49}1.77e0 & 4.43e-1 \\
    \multirow{3}[0]{*}{F8} & 5     & \cellcolor[rgb]{ .902,  .89,  .51}9.17e-2 & 4.32e-3 & \cellcolor[rgb]{ .6,  .804,  .494}8.11e-2 & 7.52e-3 & \cellcolor[rgb]{ .388,  .745,  .482}7.35e-2 & 6.34e-4 & \cellcolor[rgb]{ 1,  .914,  .518}9.86e-2 & 1.22e-1 & \cellcolor[rgb]{ .533,  .784,  .49}7.88e-2 & 2.65e-3 & \cellcolor[rgb]{ .996,  .824,  .502}1.29e-1 & 4.45e-3 & \cellcolor[rgb]{ .992,  .761,  .49}1.50e-1 & 1.74e-2 & \cellcolor[rgb]{ .973,  .412,  .42}2.67e-1 & 5.42e-2 \\
          & 10    & \cellcolor[rgb]{ .875,  .882,  .51}1.35e-1 & 4.99e-3 & \cellcolor[rgb]{ .682,  .827,  .498}1.25e-1 & 8.02e-3 & \cellcolor[rgb]{ .388,  .745,  .482}1.09e-1 & 5.52e-4 & \cellcolor[rgb]{ .996,  .827,  .502}2.98e-1 & 6.80e-3 & \cellcolor[rgb]{ 1,  .922,  .518}1.46e-1 & 4.02e-3 & \cellcolor[rgb]{ .918,  .898,  .51}1.37e-1 & 4.17e-3 & \cellcolor[rgb]{ .996,  .808,  .498}3.28e-1 & 6.06e-2 & \cellcolor[rgb]{ .973,  .412,  .42}9.74e-1 & 1.68e-1 \\
          & 15    & \cellcolor[rgb]{ 1,  .922,  .518}1.93e-1 & 1.87e-2 & \cellcolor[rgb]{ .675,  .827,  .498}1.60e-1 & 1.06e-2 & \cellcolor[rgb]{ .388,  .745,  .482}1.32e-1 & 7.36e-4 & \cellcolor[rgb]{ .992,  .757,  .486}6.05e-1 & 5.66e-5 & \cellcolor[rgb]{ .984,  .918,  .514}1.90e-1 & 9.46e-3 & \cellcolor[rgb]{ .804,  .863,  .506}1.73e-1 & 6.24e-3 & \cellcolor[rgb]{ 1,  .855,  .506}3.61e-1 & 6.02e-2 & \cellcolor[rgb]{ .973,  .412,  .42}1.45e0 & 3.20e-1 \\
    \multirow{3}[0]{*}{F9} & 5     & \cellcolor[rgb]{ .627,  .812,  .494}9.17e-2 & 6.35e-3 & \cellcolor[rgb]{ .71,  .835,  .498}9.62e-2 & 1.53e-2 & \cellcolor[rgb]{ .973,  .412,  .42}2.76e-1 & 9.57e-2 & \cellcolor[rgb]{ .659,  .824,  .498}9.34e-2 & 3.16e-1 & \cellcolor[rgb]{ .388,  .745,  .482}7.87e-2 & 5.90e-3 & \cellcolor[rgb]{ 1,  .875,  .51}1.27e-1 & 8.24e-3 & \cellcolor[rgb]{ .98,  .518,  .443}2.43e-1 & 1.13e-1 & \cellcolor[rgb]{ .988,  .69,  .475}1.87e-1 & 2.98e-2 \\
          & 10    & \cellcolor[rgb]{ .427,  .757,  .482}1.85e-1 & 1.27e-2 & \cellcolor[rgb]{ .502,  .776,  .486}1.98e-1 & 7.86e-2 & \cellcolor[rgb]{ .973,  .412,  .42}2.49e0 & 3.42e-2 & \cellcolor[rgb]{ .698,  .831,  .498}2.32e-1 & 1.35e-2 & \cellcolor[rgb]{ 1,  .91,  .518}3.38e-1 & 6.58e-2 & \cellcolor[rgb]{ .388,  .745,  .482}1.77e-1 & 8.20e-3 & \cellcolor[rgb]{ 1,  .851,  .506}5.98e-1 & 2.14e-1 & \cellcolor[rgb]{ .992,  .773,  .49}9.31e-1 & 2.16e-1 \\
          & 15    & \cellcolor[rgb]{ .518,  .78,  .486}2.25e-1 & 6.25e-2 & \cellcolor[rgb]{ .431,  .757,  .482}1.76e-1 & 1.21e-1 & \cellcolor[rgb]{ .973,  .412,  .42}3.34e0 & 5.50e0 & \cellcolor[rgb]{ 1,  .898,  .514}6.40e-1 & 2.67e-4 & \cellcolor[rgb]{ .737,  .843,  .502}3.48e-1 & 1.45e-1 & \cellcolor[rgb]{ .388,  .745,  .482}1.51e-1 & 5.40e-3 & \cellcolor[rgb]{ .98,  .537,  .447}2.65e0 & 4.63e0 & \cellcolor[rgb]{ .996,  .808,  .498}1.14e0 & 2.02e-1 \\
    \multirow{3}[0]{*}{F10} & 5     & \cellcolor[rgb]{ .529,  .784,  .49}3.87e-1 & 1.45e-2 & \cellcolor[rgb]{ .937,  .902,  .514}4.43e-1 & 1.82e-2 & \cellcolor[rgb]{ .388,  .745,  .482}3.67e-1 & 1.87e-3 & \cellcolor[rgb]{ .996,  .78,  .49}5.02e-1 & 9.44e-1 & \cellcolor[rgb]{ .431,  .757,  .482}3.73e-1 & 8.63e-3 & \cellcolor[rgb]{ 1,  .867,  .51}4.71e-1 & 1.09e-2 & \cellcolor[rgb]{ 1,  .902,  .514}4.59e-1 & 3.35e-2 & \cellcolor[rgb]{ .973,  .412,  .42}6.30e-1 & 7.66e-2 \\
          & 10    & \cellcolor[rgb]{ .714,  .839,  .498}1.02e0 & 2.98e-2 & \cellcolor[rgb]{ .973,  .412,  .42}1.40e0 & 5.19e-2 & \cellcolor[rgb]{ .388,  .745,  .482}9.66e-1 & 1.94e-2 & \cellcolor[rgb]{ .749,  .847,  .502}1.02e0 & 9.98e-1 & \cellcolor[rgb]{ .922,  .898,  .51}1.05e0 & 3.73e-2 & \cellcolor[rgb]{ .992,  .714,  .478}1.20e0 & 5.94e-2 & \cellcolor[rgb]{ 1,  .906,  .518}1.07e0 & 4.19e-2 & \cellcolor[rgb]{ .98,  .494,  .439}1.35e0 & 9.24e-2 \\
          & 15    & \cellcolor[rgb]{ .388,  .745,  .482}1.38e0 & 3.91e-2 & \cellcolor[rgb]{ .973,  .412,  .42}2.12e0 & 1.05e-1 & \cellcolor[rgb]{ .62,  .812,  .494}1.50e0 & 4.59e-2 & \cellcolor[rgb]{ .871,  .882,  .51}1.63e0 & 9.98e-1 & \cellcolor[rgb]{ .659,  .824,  .498}1.52e0 & 6.54e-2 & \cellcolor[rgb]{ .984,  .596,  .455}1.96e0 & 8.53e-2 & \cellcolor[rgb]{ .996,  .843,  .506}1.76e0 & 5.59e-1 & \cellcolor[rgb]{ .98,  .549,  .447}2.00e0 & 6.39e-2 \\
    \multirow{3}[0]{*}{F11} & 5     & \cellcolor[rgb]{ .388,  .745,  .482}3.89e-1 & 1.93e-3 & \cellcolor[rgb]{ .973,  .412,  .42}4.46e0 & 9.22e-3 & \cellcolor[rgb]{ .541,  .788,  .49}4.78e-1 & 9.35e-3 & \cellcolor[rgb]{ .855,  .878,  .506}6.58e-1 & 9.95e-1 & \cellcolor[rgb]{ .816,  .867,  .506}6.35e-1 & 1.26e-2 & \cellcolor[rgb]{ 1,  .914,  .518}8.25e-1 & 1.72e-2 & \cellcolor[rgb]{ 1,  .882,  .514}1.04e0 & 5.00e-1 & \cellcolor[rgb]{ .988,  .694,  .475}2.42e0 & 8.17e-1 \\
          & 10    & \cellcolor[rgb]{ .388,  .745,  .482}1.15e0 & 2.26e-2 & \cellcolor[rgb]{ .51,  .776,  .486}1.47e0 & 4.80e-2 & \cellcolor[rgb]{ .4,  .745,  .482}1.18e0 & 2.21e-2 & \cellcolor[rgb]{ .996,  .796,  .494}4.76e0 & 9.96e-1 & \cellcolor[rgb]{ 1,  .91,  .518}3.00e0 & 1.28e0 & \cellcolor[rgb]{ .914,  .894,  .51}2.56e0 & 4.39e-1 & \cellcolor[rgb]{ .992,  .733,  .482}5.70e0 & 6.26e-1 & \cellcolor[rgb]{ .973,  .412,  .42}1.06e1 & 2.20e0 \\
          & 15    & \cellcolor[rgb]{ .545,  .788,  .49}1.43e0 & 3.71e-2 & \cellcolor[rgb]{ .667,  .824,  .498}2.17e0 & 4.47e-2 & \cellcolor[rgb]{ .58,  .8,  .49}1.65e0 & 4.79e-2 & \cellcolor[rgb]{ .996,  .827,  .502}6.21e0 & 9.95e-1 & \cellcolor[rgb]{ .992,  .718,  .478}8.52e0 & 2.79e0 & \cellcolor[rgb]{ .388,  .745,  .482}4.68e-1 & 6.22e-1 & \cellcolor[rgb]{ .98,  .514,  .439}1.29e1 & 1.60e0 & \cellcolor[rgb]{ .973,  .412,  .42}1.50e1 & 2.36e0 \\
    \multirow{3}[0]{*}{F12} & 5     & \cellcolor[rgb]{ .757,  .851,  .502}9.36e-1 & 3.43e-3 & \cellcolor[rgb]{ 1,  .859,  .506}9.63e-1 & 9.78e-3 & \cellcolor[rgb]{ .827,  .871,  .506}9.37e-1 & 8.21e-3 & \cellcolor[rgb]{ .655,  .82,  .494}9.35e-1 & 7.76e-1 & \cellcolor[rgb]{ .388,  .745,  .482}9.33e-1 & 2.51e-3 & \cellcolor[rgb]{ .973,  .412,  .42}1.12e0 & 7.95e-3 & \cellcolor[rgb]{ 1,  .918,  .518}9.40e-1 & 8.33e-3 & \cellcolor[rgb]{ .996,  .843,  .506}9.68e-1 & 9.86e-3 \\
          & 10    & \cellcolor[rgb]{ .988,  .678,  .475}4.60e0 & 1.81e-2 & \cellcolor[rgb]{ .431,  .757,  .482}4.17e0 & 2.26e-2 & \cellcolor[rgb]{ .388,  .745,  .482}4.15e0 & 8.14e-3 & \cellcolor[rgb]{ .988,  .69,  .475}4.60e0 & 8.98e-1 & \cellcolor[rgb]{ .718,  .839,  .498}4.35e0 & 2.06e-2 & \cellcolor[rgb]{ .973,  .412,  .42}4.69e0 & 1.44e-2 & \cellcolor[rgb]{ .996,  .824,  .502}4.55e0 & 2.34e-1 & \cellcolor[rgb]{ .945,  .906,  .514}4.49e0 & 4.47e-2 \\
          & 15    & \cellcolor[rgb]{ 1,  .89,  .514}7.73e0 & 8.12e-2 & \cellcolor[rgb]{ .388,  .745,  .482}7.19e0 & 6.65e-2 & \cellcolor[rgb]{ .447,  .761,  .482}7.24e0 & 1.15e-1 & \cellcolor[rgb]{ .996,  .847,  .506}7.80e0 & 8.26e-1 & \cellcolor[rgb]{ .855,  .878,  .506}7.56e0 & 1.58e-1 & \cellcolor[rgb]{ .929,  .898,  .51}7.62e0 & 1.66e-1 & \cellcolor[rgb]{ .973,  .412,  .42}8.48e0 & 3.55e-1 & \cellcolor[rgb]{ .976,  .463,  .431}8.40e0 & 1.47e-1 \\
    \multirow{3}[0]{*}{F13} & 5     & \cellcolor[rgb]{ .514,  .78,  .486}8.83e-2 & 1.25e-2 & \cellcolor[rgb]{ .388,  .745,  .482}7.79e-2 & 1.27e-2 & \cellcolor[rgb]{ .969,  .91,  .514}1.25e-1 & 4.12e-2 & \cellcolor[rgb]{ 1,  .922,  .518}1.30e-1 & 2.61e-1 & \cellcolor[rgb]{ .906,  .894,  .51}1.20e-1 & 1.19e-2 & \cellcolor[rgb]{ 1,  .918,  .518}1.30e-1 & 6.74e-3 & \cellcolor[rgb]{ 1,  .867,  .51}1.63e-1 & 1.72e-2 & \cellcolor[rgb]{ .973,  .412,  .42}4.51e-1 & 9.00e-2 \\
          & 10    & \cellcolor[rgb]{ .494,  .773,  .486}1.12e-1 & 5.25e-2 & \cellcolor[rgb]{ .388,  .745,  .482}9.98e-2 & 6.79e-3 & \cellcolor[rgb]{ .58,  .8,  .49}1.22e-1 & 7.88e-3 & \cellcolor[rgb]{ .988,  .647,  .467}3.23e-1 & 1.32e-1 & \cellcolor[rgb]{ .996,  .839,  .502}2.17e-1 & 1.15e-2 & \cellcolor[rgb]{ .58,  .8,  .49}1.22e-1 & 7.21e-3 & \cellcolor[rgb]{ .996,  .816,  .498}2.30e-1 & 2.00e-2 & \cellcolor[rgb]{ .973,  .412,  .42}4.52e-1 & 9.87e-2 \\
          & 15    & \cellcolor[rgb]{ .533,  .784,  .49}1.43e-1 & 6.35e-2 & \cellcolor[rgb]{ .388,  .745,  .482}1.25e-1 & 1.18e-2 & \cellcolor[rgb]{ .459,  .765,  .486}1.34e-1 & 7.74e-3 & \cellcolor[rgb]{ .988,  .639,  .467}4.43e-1 & 5.24e-2 & \cellcolor[rgb]{ 1,  .867,  .51}2.49e-1 & 3.90e-2 & \cellcolor[rgb]{ .588,  .8,  .49}1.50e-1 & 9.25e-3 & \cellcolor[rgb]{ 1,  .863,  .51}2.51e-1 & 3.13e-2 & \cellcolor[rgb]{ .973,  .412,  .42}6.38e-1 & 7.86e-2 \\
    \multirow{3}[0]{*}{F14} & 5     & \cellcolor[rgb]{ .882,  .886,  .51}3.42e-1 & 2.47e-2 & \cellcolor[rgb]{ .388,  .745,  .482}2.26e-1 & 4.65e-2 & \cellcolor[rgb]{ .988,  .686,  .475}5.30e-1 & 7.32e-2 & \cellcolor[rgb]{ .929,  .898,  .51}3.52e-1 & 6.01e-1 & \cellcolor[rgb]{ .918,  .898,  .51}3.50e-1 & 2.28e-2 & \cellcolor[rgb]{ 1,  .898,  .514}3.85e-1 & 4.24e-2 & \cellcolor[rgb]{ .973,  .412,  .42}7.14e-1 & 2.36e-1 & \cellcolor[rgb]{ .988,  .69,  .475}5.26e-1 & 8.67e-2 \\
          & 10    & \cellcolor[rgb]{ .388,  .745,  .482}5.49e-1 & 7.73e-2 & \cellcolor[rgb]{ 1,  .922,  .518}8.49e-1 & 1.69e-1 & \cellcolor[rgb]{ .973,  .412,  .42}2.45e1 & 3.53e1 & \cellcolor[rgb]{ .945,  .906,  .514}6.95e-1 & 5.24e-1 & \cellcolor[rgb]{ .404,  .749,  .482}5.54e-1 & 5.09e-2 & \cellcolor[rgb]{ .875,  .882,  .51}6.76e-1 & 8.35e-2 & \cellcolor[rgb]{ 1,  .89,  .514}2.34e0 & 1.31e0 & \cellcolor[rgb]{ 1,  .922,  .518}7.22e-1 & 5.71e-2 \\
          & 15    & \cellcolor[rgb]{ .439,  .757,  .482}6.22e-1 & 1.43e-1 & \cellcolor[rgb]{ 1,  .914,  .518}1.12e0 & 1.45e-1 & \cellcolor[rgb]{ .973,  .412,  .42}1.80e1 & 1.60e1 & \cellcolor[rgb]{ .753,  .851,  .502}7.28e-1 & 4.39e-1 & \cellcolor[rgb]{ .431,  .757,  .482}6.19e-1 & 1.33e-1 & \cellcolor[rgb]{ .388,  .745,  .482}6.04e-1 & 6.46e-2 & \cellcolor[rgb]{ 1,  .91,  .518}1.32e0 & 2.61e-1 & \cellcolor[rgb]{ 1,  .922,  .518}8.92e-1 & 1.30e-1 \\
    \multirow{3}[1]{*}{F15} & 5     & \cellcolor[rgb]{ .545,  .788,  .49}3.06e-1 & 4.33e-2 & \cellcolor[rgb]{ .388,  .745,  .482}2.33e-1 & 5.41e-2 & \cellcolor[rgb]{ .976,  .459,  .431}9.52e-1 & 4.41e-2 & \cellcolor[rgb]{ .792,  .859,  .502}4.21e-1 & 4.28e-2 & \cellcolor[rgb]{ .678,  .827,  .498}3.68e-1 & 4.02e-2 & \cellcolor[rgb]{ .996,  .82,  .498}6.13e-1 & 2.57e-2 & \cellcolor[rgb]{ .973,  .412,  .42}9.94e-1 & 9.40e-2 & \cellcolor[rgb]{ .996,  .816,  .498}6.17e-1 & 4.06e-2 \\
          & 10    & \cellcolor[rgb]{ .592,  .804,  .494}8.99e-1 & 9.81e-2 & \cellcolor[rgb]{ 1,  .922,  .518}1.03e0 & 2.05e-1 & \cellcolor[rgb]{ .973,  .412,  .42}9.19e0 & 1.35e1 & \cellcolor[rgb]{ 1,  .922,  .518}9.97e-1 & 4.43e-7 & \cellcolor[rgb]{ .969,  .91,  .514}9.84e-1 & 4.72e-2 & \cellcolor[rgb]{ .388,  .745,  .482}8.52e-1 & 5.04e-2 & \cellcolor[rgb]{ 1,  .89,  .514}1.53e0 & 2.67e-1 & \cellcolor[rgb]{ .957,  .91,  .514}9.81e-1 & 6.80e-2 \\
          & 15    & \cellcolor[rgb]{ .388,  .745,  .482}1.02e0 & 1.35e-1 & \cellcolor[rgb]{ .996,  .8,  .498}1.93e0 & 4.95e-1 & \cellcolor[rgb]{ 1,  .902,  .514}1.42e0 & 9.17e-2 & \cellcolor[rgb]{ .647,  .82,  .494}1.14e0 & 5.14e-11 & \cellcolor[rgb]{ .431,  .757,  .482}1.04e0 & 4.51e-2 & \cellcolor[rgb]{ 1,  .914,  .518}1.35e0 & 7.91e-2 & \cellcolor[rgb]{ .973,  .412,  .42}3.89e0 & 2.03e0 & \cellcolor[rgb]{ .906,  .894,  .51}1.27e0 & 5.40e-2 \\
    \midrule
    \multirow{6}[2]{*}{Friedman} & 5     & \multicolumn{2}{c}{\cellcolor[rgb]{ .365,  .549,  .78}2.57} & \multicolumn{2}{c}{\cellcolor[rgb]{ .729,  .804,  .906}3.53} & \multicolumn{2}{c}{\cellcolor[rgb]{ .843,  .886,  .949}3.83} & \multicolumn{2}{c}{\cellcolor[rgb]{ .988,  .886,  .898}4.60} & \multicolumn{2}{c}{\cellcolor[rgb]{ .353,  .541,  .776}\textbf{2.53}} & \multicolumn{2}{c}{\cellcolor[rgb]{ .976,  .424,  .431}6.30} & \multicolumn{2}{c}{\cellcolor[rgb]{ .976,  .424,  .431}6.30} & \multicolumn{2}{c}{\cellcolor[rgb]{ .973,  .412,  .42}6.33} \\
          & 10    & \multicolumn{2}{c}{\cellcolor[rgb]{ .353,  .541,  .776}\textbf{2.50}} & \multicolumn{2}{c}{\cellcolor[rgb]{ .639,  .745,  .878}3.37} & \multicolumn{2}{c}{\cellcolor[rgb]{ .776,  .839,  .925}3.77} & \multicolumn{2}{c}{\cellcolor[rgb]{ .984,  .769,  .78}5.17} & \multicolumn{2}{c}{\cellcolor[rgb]{ .831,  .875,  .941}3.93} & \multicolumn{2}{c}{\cellcolor[rgb]{ .988,  .855,  .867}4.87} & \multicolumn{2}{c}{\cellcolor[rgb]{ .973,  .412,  .42}6.40} & \multicolumn{2}{c}{\cellcolor[rgb]{ .976,  .529,  .537}6.00} \\
          & 15    & \multicolumn{2}{c}{\cellcolor[rgb]{ .506,  .647,  .827}\textbf{2.80}} & \multicolumn{2}{c}{\cellcolor[rgb]{ .698,  .784,  .898}3.40} & \multicolumn{2}{c}{\cellcolor[rgb]{ .988,  .988,  1}4.30} & \multicolumn{2}{c}{\cellcolor[rgb]{ .984,  .741,  .753}5.20} & \multicolumn{2}{c}{\cellcolor[rgb]{ .667,  .761,  .886}3.30} & \multicolumn{2}{c}{\cellcolor[rgb]{ .988,  .91,  .918}4.60} & \multicolumn{2}{c}{\cellcolor[rgb]{ .973,  .412,  .42}6.40} & \multicolumn{2}{c}{\cellcolor[rgb]{ .976,  .525,  .533}6.00} \\
          & partial & \multicolumn{2}{c}{\cellcolor[rgb]{ .353,  .541,  .776}\textbf{2.31}} & \multicolumn{2}{c}{\cellcolor[rgb]{ .502,  .647,  .827}2.87} & \multicolumn{2}{c}{\cellcolor[rgb]{ .886,  .914,  .961}4.28} & \multicolumn{2}{c}{\cellcolor[rgb]{ .984,  .761,  .773}5.39} & \multicolumn{2}{c}{\cellcolor[rgb]{ .694,  .78,  .894}3.57} & \multicolumn{2}{c}{\cellcolor[rgb]{ .988,  .875,  .886}5.02} & \multicolumn{2}{c}{\cellcolor[rgb]{ .973,  .412,  .42}6.52} & \multicolumn{2}{c}{\cellcolor[rgb]{ .98,  .561,  .573}6.04} \\
          & full  & \multicolumn{2}{c}{\cellcolor[rgb]{ .353,  .541,  .776}\textbf{2.97}} & \multicolumn{2}{c}{\cellcolor[rgb]{ .988,  .961,  .973}4.39} & \multicolumn{2}{c}{\cellcolor[rgb]{ .714,  .792,  .902}3.72} & \multicolumn{2}{c}{\cellcolor[rgb]{ .941,  .953,  .98}4.19} & \multicolumn{2}{c}{\cellcolor[rgb]{ .353,  .541,  .776}\textbf{2.97}} & \multicolumn{2}{c}{\cellcolor[rgb]{ .98,  .584,  .592}5.61} & \multicolumn{2}{c}{\cellcolor[rgb]{ .976,  .475,  .482}5.97} & \multicolumn{2}{c}{\cellcolor[rgb]{ .973,  .412,  .42}6.17} \\
          & overall & \multicolumn{2}{c}{\cellcolor[rgb]{ .353,  .541,  .776}\textbf{2.62}} & \multicolumn{2}{c}{\cellcolor[rgb]{ .627,  .733,  .871}3.43} & \multicolumn{2}{c}{\cellcolor[rgb]{ .812,  .863,  .937}3.97} & \multicolumn{2}{c}{\cellcolor[rgb]{ .984,  .835,  .843}4.99} & \multicolumn{2}{c}{\cellcolor[rgb]{ .569,  .69,  .851}3.26} & \multicolumn{2}{c}{\cellcolor[rgb]{ .984,  .753,  .765}5.26} & \multicolumn{2}{c}{\cellcolor[rgb]{ .973,  .412,  .42}6.37} & \multicolumn{2}{c}{\cellcolor[rgb]{ .976,  .49,  .502}6.11} \\
    \midrule
    \multirow{6}[2]{*}{t-test} & 5     & \multicolumn{2}{c}{\multirow{6}[2]{*}{~}} & \multicolumn{2}{c}{5/3/7} & \multicolumn{2}{c}{\textbf{8/2/5}} & \multicolumn{2}{c}{\textbf{3/12/0}} & \multicolumn{2}{c}{\textbf{7/2/6}} & \multicolumn{2}{c}{\textbf{15/0/0}} & \multicolumn{2}{c}{\textbf{13/2/0}} & \multicolumn{2}{c}{\textbf{13/1/1}} \\
          & 10    & \multicolumn{2}{c}{} & \multicolumn{2}{c}{\textbf{7/3/5}} & \multicolumn{2}{c}{\textbf{8/3/4}} & \multicolumn{2}{c}{\textbf{9/5/1}} & \multicolumn{2}{c}{\textbf{11/1/3}} & \multicolumn{2}{c}{\textbf{10/5/0}} & \multicolumn{2}{c}{\textbf{11/3/1}} & \multicolumn{2}{c}{\textbf{11/2/2}} \\
          & 15    & \multicolumn{2}{c}{} & \multicolumn{2}{c}{\textbf{8/1/6}} & \multicolumn{2}{c}{\textbf{10/2/3}} & \multicolumn{2}{c}{\textbf{10/4/1}} & \multicolumn{2}{c}{\textbf{8/5/2}} & \multicolumn{2}{c}{\textbf{8/3/4}} & \multicolumn{2}{c}{\textbf{14/1/0}} & \multicolumn{2}{c}{\textbf{11/3/1}} \\
          & partial & \multicolumn{2}{c}{} & \multicolumn{2}{c}{8/11/8} & \multicolumn{2}{c}{\textbf{16/6/5}} & \multicolumn{2}{c}{\textbf{20/7/0}} & \multicolumn{2}{c}{\textbf{18/5/4}} & \multicolumn{2}{c}{\textbf{20/4/3}} & \multicolumn{2}{c}{\textbf{24/3/0}} & \multicolumn{2}{c}{\textbf{23/2/2}} \\
          & full  & \multicolumn{2}{c}{} & \multicolumn{2}{c}{\textbf{11/1/6}} & \multicolumn{2}{c}{\textbf{11/0/7}} & \multicolumn{2}{c}{\textbf{4/14/0}} & \multicolumn{2}{c}{0/18/0} & \multicolumn{2}{c}{\textbf{13/3/2}} & \multicolumn{2}{c}{\textbf{12/5/1}} & \multicolumn{2}{c}{\textbf{14/2/2}} \\
          & overall & \multicolumn{2}{c}{} & \multicolumn{2}{c}{\textbf{19/9/17}} & \multicolumn{2}{c}{\textbf{26/7/12}} & \multicolumn{2}{c}{\textbf{22/23/0}} & \multicolumn{2}{c}{\textbf{23/10/12}} & \multicolumn{2}{c}{\textbf{34/6/5}} & \multicolumn{2}{c}{\textbf{40/3/2}} & \multicolumn{2}{c}{\textbf{38/3/4}} \\
    \bottomrule
    \bottomrule
	\multicolumn{18}{m{0.6\textwidth}}{For Friedman test, $\alpha = 0.05$, $p \ll \alpha$; For $t$-test, $\alpha = 0.05$. The results of the $t$-tests are provided as ``$l/u/g$'', where $l$ represents the mean value of the proposed algorithm, for a certain test case, is significantly \emph{less} than the compared algorithm, $g$ represents a \emph{greater} result, and $u$ represents that it is \emph{uncertain} to say whether the mean value of the proposed algorithm \algoabbr{} is greater or less than the compared algorithm.} \\
    \end{tabular}%
  \label{tab:IGD}%
\end{table*}%

\par
To observe the performance of \algoabbr{} more intuitively, we provide the comparison graph between the obtained PF with median IGD among the independent runs and the ground truth PF on $4$ test cases with partial FOS. We can see from these figures that, apart from the noise caused by evolution, \algoabbr{} obtains intuitively satisfying performance.
\begin{figure*}
\centering
\subfloat[MaF1, $M = 5$]{
\captionsetup{justification = centering}
\includegraphics[width=0.45\textwidth]{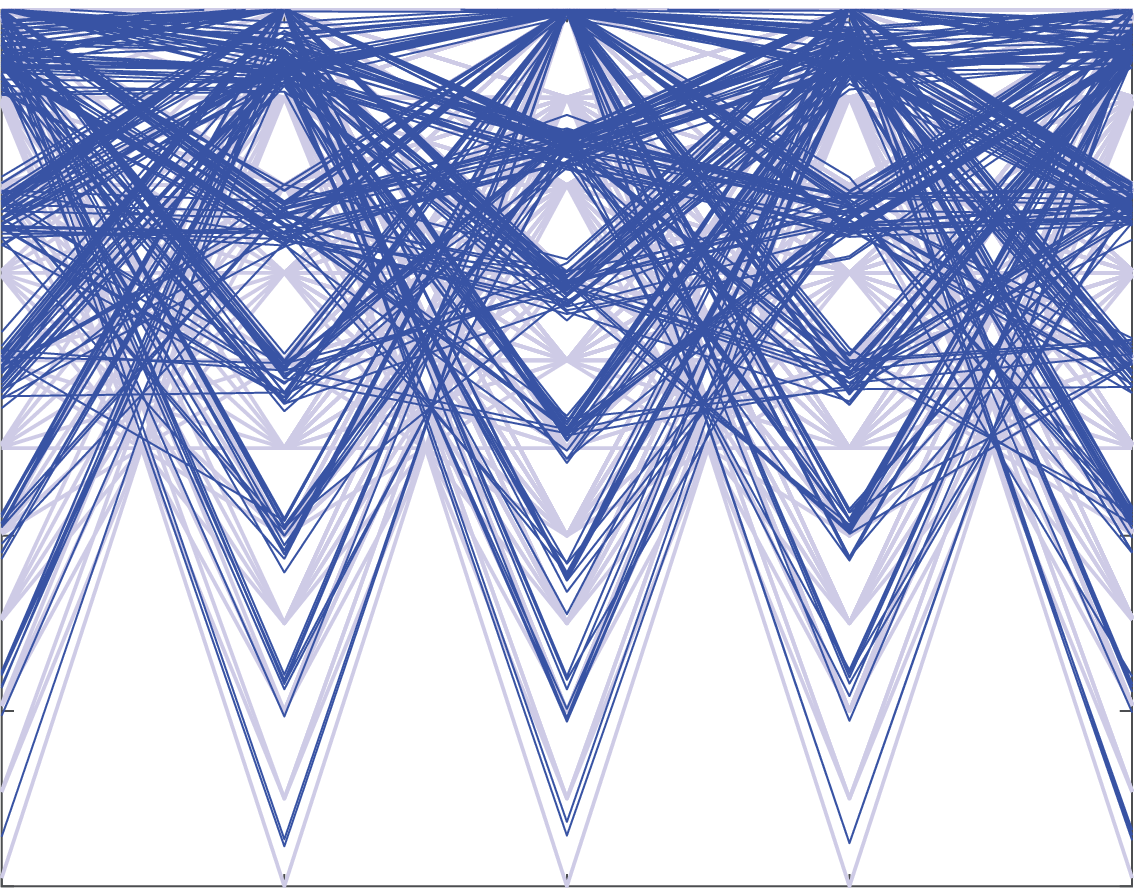}}
\hfill
\subfloat[MaF2, $M = 5$]{
\captionsetup{justification = centering}
\includegraphics[width=0.45\textwidth]{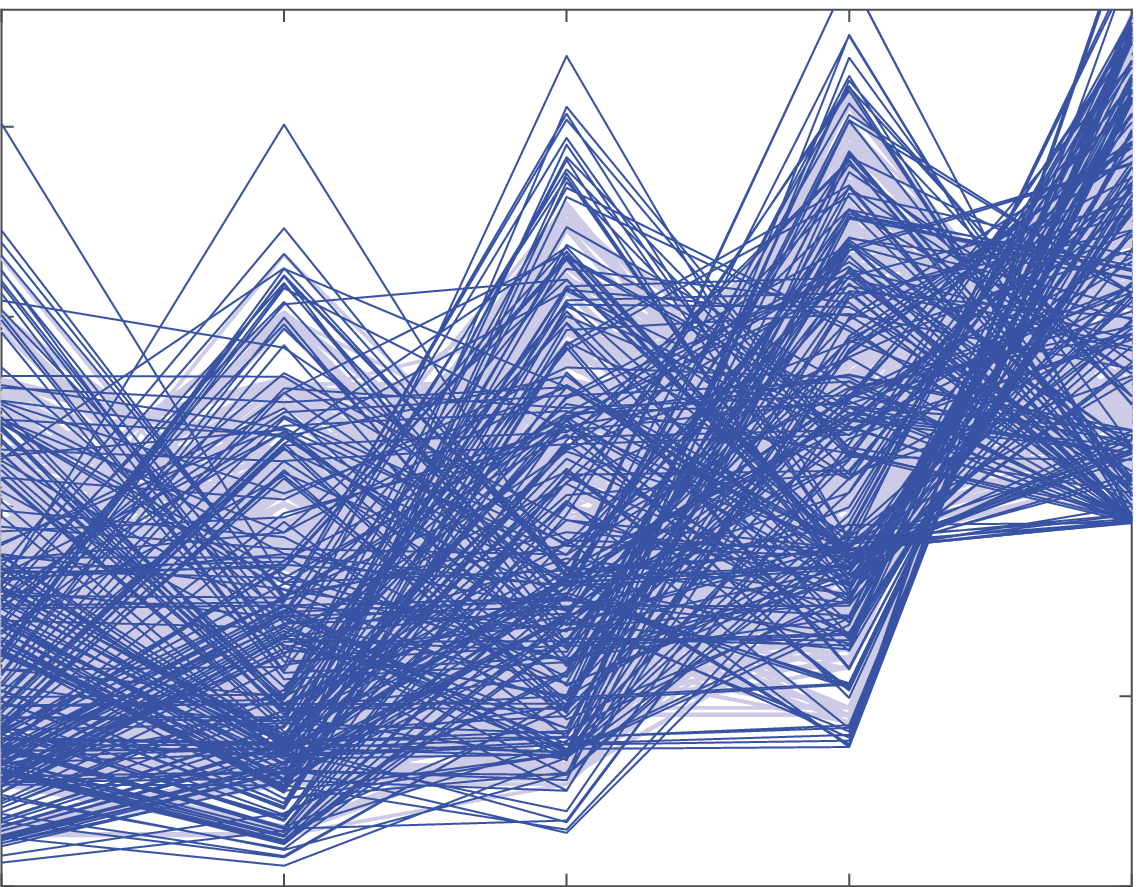}}

\subfloat[MaF6, $M = 5$]{
\captionsetup{justification = centering}
\includegraphics[width=0.45\textwidth]{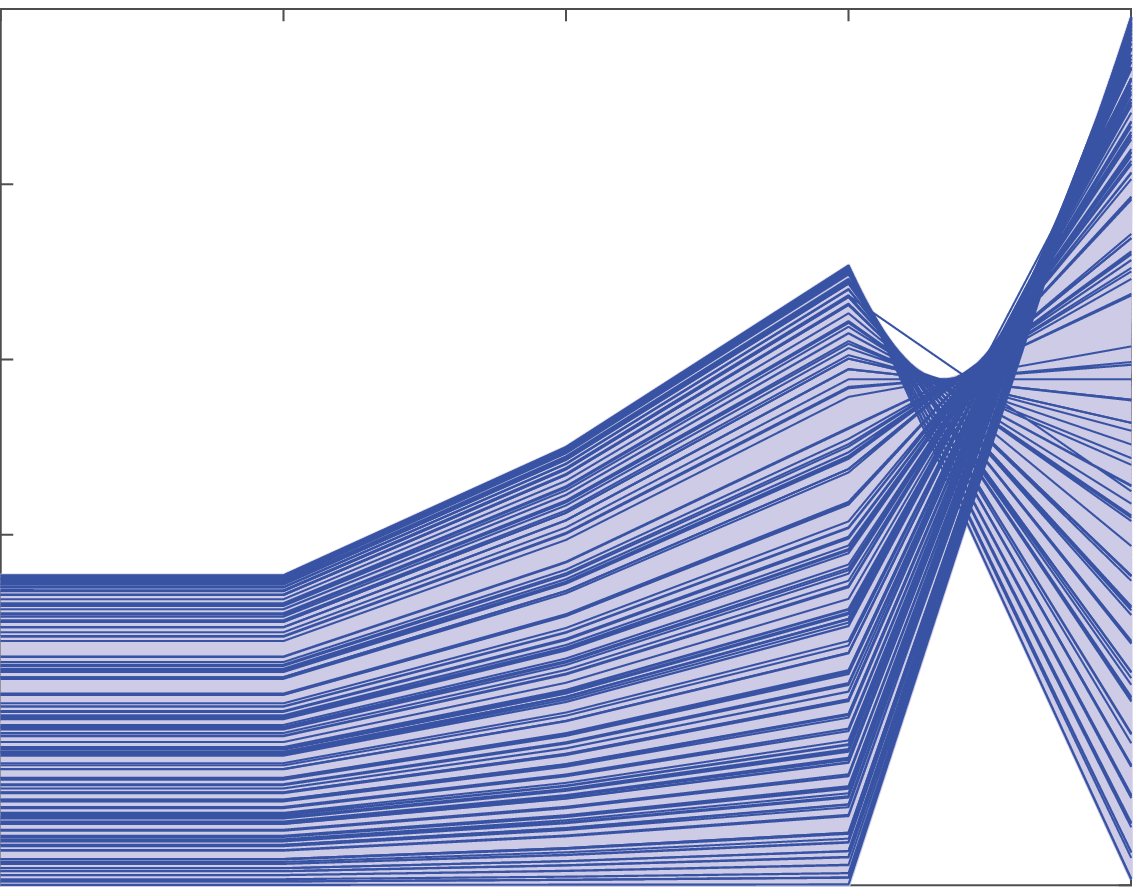}}
\hfill
\subfloat[MaF13, $M = 5$]{
\captionsetup{justification = centering}
\includegraphics[width=0.45\textwidth]{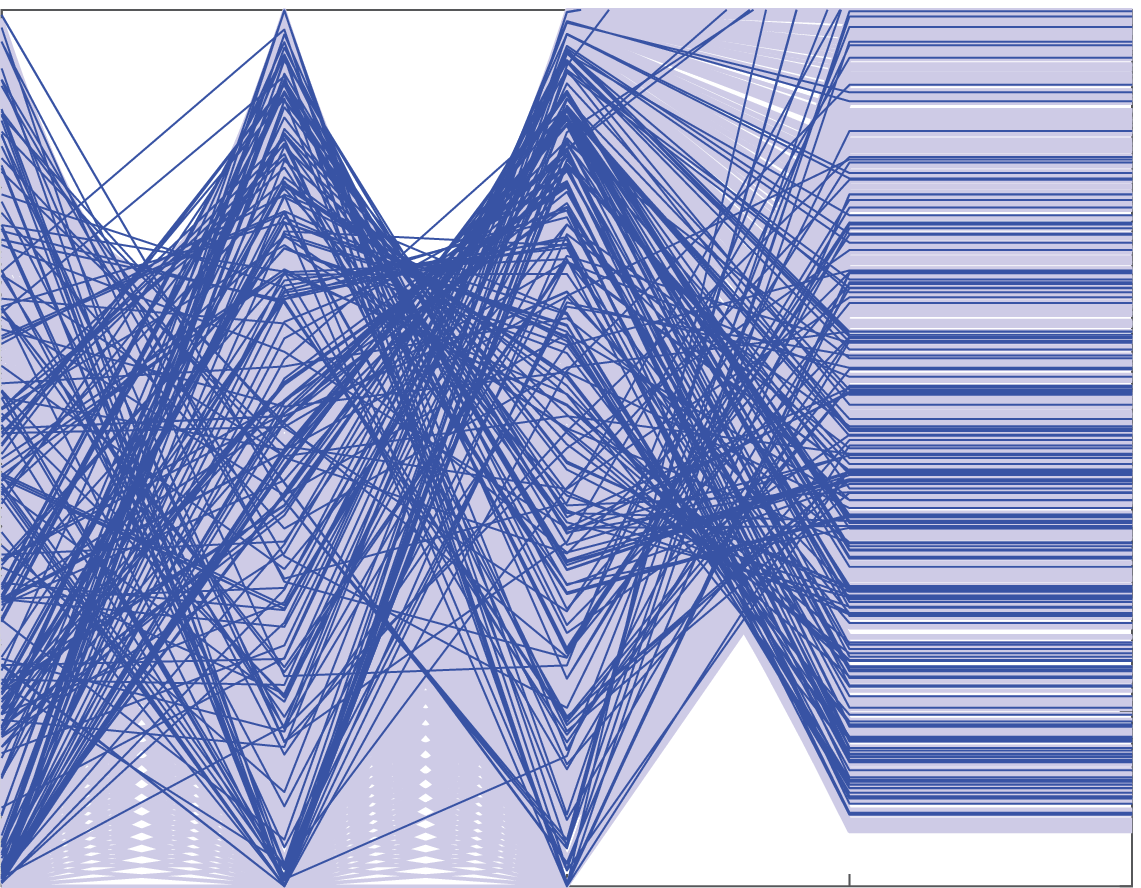}}
\caption{Visualization of the obtained PF with median IGD: for comparison, the true PF is also plotted in the parallel coordinates with lighter color (as the background).}
\label{fig:cover}
\end{figure*}
\par
To specify the performance on each type of algorithms, we offer two dichotomies for the test cases, first of which is based on the number of objectives and second is based on the feasibility of the FOS. We conduct Friedman tests and $t$-tests on the partitioned sets of the test cases and present the statistics in Tab. \ref{tab:IGD}. Also, to intuitively compare the performance, we convert the results of the Friedman test into a radar diagram presented in Fig. \ref{fig:radar}. From the categorical observations we find that
\begin{enumerate}
\item \algoabbr{} has competitive performance within all partitions of test cases, ranking at least the 2nd among the state-of-the-art algorithms;
\item \algoabbr{} has leading performance in test cases with partial FOS. This shows its effectiveness.
\item \algoabbr{} and CLIA achives highest performance on the test cases with full FOS, this shows the effectiveness of the selection operator cascade clustering;
\item \algoabbr{} has arguably similar performance (from $t$-test) when compared to CVEA3, which strongly shows the competitiveness of its performance.
\end{enumerate}
\begin{figure*}
\centering
\captionsetup{justification = centering}
\includegraphics[width=0.45\textwidth]{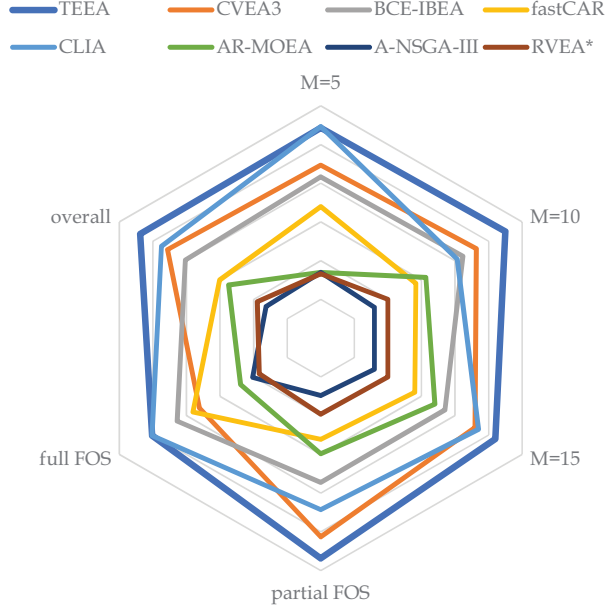}
\caption{Radar visualization of the performance on different types of test cases: the wider the distribution of the included areas, the better the performance. For each dimension of the radar diagram, the values are $|A| - r_F$, where $|A| = 8$ is the number of compared algorithms and $r_F$ is the mean rank given by the Friedman tests.}
\label{fig:radar}
\end{figure*}
\subsection{Comparative Tests on Difficult Problem with Infeasible Parts in Objective Spaces}
To validate the effectiveness of the main contribution of this work, we extend the dimensionality of the solution space of the MaF15 problem into $D \in \{100, 200, 300, 400, 500\}$ for $M \in \{5, 10\}$, constituting $10$ extremely difficult test cases. These test cases simulate the scenarios in which the population hardly reaches vicinity of the true PF and the objective space is with infeasible parts. We compare \algoabbr{} with the reference vector based opponents with adaptation, including fastCAR \cite{zhao2018fast}, CLIA \cite{ge2018interacting}, AR-MOEA \cite{tian2017indicator}, A-NSGA-III \cite{jain2014evolutionary} and RVEA* \cite{cheng2016reference}. The results are given in Tab. \ref{tab:difficult}.
\begin{table*}[htbp]
\setlength{\tabcolsep}{1.5pt}
\scriptsize
  \centering
  \caption{Results on A Difficult Problem with Infeasible Parts in the Objective Space}
    \begin{tabular}{cccccccccccc}
    \toprule
    \toprule
    \multicolumn{2}{c}{MaF15} & \multicolumn{2}{c}{TEEA} & \multicolumn{2}{c}{fastCAR} & \multicolumn{2}{c}{CLIA} & \multicolumn{2}{c}{A-NSGA-III} & \multicolumn{2}{c}{RVEA*} \\
    M     & D     & Mean  & Std   & Mean  & Std   & Mean  & Std   & Mean  & Std   & Mean  & Std \\
    \midrule
    \multirow{5}[1]{*}{5} & 100   & \cellcolor[rgb]{ .388,  .745,  .482}3.06e-1 & 3.25e-3 & \cellcolor[rgb]{ 1,  .922,  .518}4.21e-1 & 4.28e-2 & \cellcolor[rgb]{ .714,  .839,  .498}3.68e-1 & 4.02e-2 & \cellcolor[rgb]{ .973,  .412,  .42}9.94e-1 & 9.40e-2 & \cellcolor[rgb]{ .992,  .749,  .486}6.17e-1 & 4.06e-2 \\
          & 200   & \cellcolor[rgb]{ .388,  .745,  .482}3.08e-1 & 2.05e-2 & \cellcolor[rgb]{ 1,  .922,  .518}4.65e-1 & 7.68e-2 & \cellcolor[rgb]{ .624,  .812,  .494}3.69e-1 & 4.03e-2 & \cellcolor[rgb]{ .973,  .412,  .42}9.22e-1 & 8.51e-3 & \cellcolor[rgb]{ .992,  .714,  .478}6.54e-1 & 3.05e-2 \\
          & 300   & \cellcolor[rgb]{ .388,  .745,  .482}3.08e-1 & 1.85e-2 & \cellcolor[rgb]{ 1,  .922,  .518}4.77e-1 & 8.25e-2 & \cellcolor[rgb]{ .608,  .808,  .494}3.69e-1 & 5.79e-2 & \cellcolor[rgb]{ .973,  .412,  .42}8.75e-1 & 5.44e-2 & \cellcolor[rgb]{ .988,  .69,  .475}6.60e-1 & 2.48e-2 \\
          & 400   & \cellcolor[rgb]{ .388,  .745,  .482}3.20e-1 & 1.43e-2 & \cellcolor[rgb]{ 1,  .922,  .518}4.78e-1 & 6.35e-2 & \cellcolor[rgb]{ .588,  .8,  .49}3.72e-1 & 2.25e-2 & \cellcolor[rgb]{ .973,  .412,  .42}8.88e-1 & 6.45e-2 & \cellcolor[rgb]{ .988,  .682,  .475}6.71e-1 & 2.94e-2 \\
          & 500   & \cellcolor[rgb]{ .388,  .745,  .482}3.28e-1 & 7.65e-2 & \cellcolor[rgb]{ 1,  .922,  .518}4.90e-1 & 4.52e-2 & \cellcolor[rgb]{ .569,  .796,  .49}3.76e-1 & 5.85e-2 & \cellcolor[rgb]{ .973,  .412,  .42}8.93e-1 & 5.33e-2 & \cellcolor[rgb]{ .992,  .733,  .482}6.40e-1 & 3.45e-2 \\
    \multirow{5}[1]{*}{10} & 100   & \cellcolor[rgb]{ .388,  .745,  .482}8.85e-1 & 3.25e-2 & \cellcolor[rgb]{ 1,  .89,  .514}9.98e-1 & 3.85e-7 & \cellcolor[rgb]{ 1,  .922,  .518}9.67e-1 & 3.25e-2 & \cellcolor[rgb]{ .973,  .412,  .42}1.43e0 & 2.13e-1 & \cellcolor[rgb]{ .737,  .843,  .502}9.32e-1 & 7.25e-2 \\
          & 200   & \cellcolor[rgb]{ .388,  .745,  .482}9.02e-1 & 4.67e-2 & \cellcolor[rgb]{ 1,  .91,  .518}9.97e-1 & 4.43e-7 & \cellcolor[rgb]{ 1,  .922,  .518}9.84e-1 & 4.72e-2 & \cellcolor[rgb]{ .973,  .412,  .42}1.53e0 & 2.67e-1 & \cellcolor[rgb]{ .976,  .914,  .514}9.81e-1 & 6.80e-2 \\
          & 300   & \cellcolor[rgb]{ .388,  .745,  .482}9.06e-1 & 4.32e-2 & \cellcolor[rgb]{ 1,  .898,  .514}1.03e0 & 1.33e-2 & \cellcolor[rgb]{ 1,  .922,  .518}9.95e-1 & 6.85e-2 & \cellcolor[rgb]{ .973,  .412,  .42}1.65e0 & 3.57e-1 & \cellcolor[rgb]{ .678,  .827,  .498}9.49e-1 & 1.21e-1 \\
          & 400   & \cellcolor[rgb]{ .388,  .745,  .482}9.10e-1 & 5.25e-2 & \cellcolor[rgb]{ .992,  .718,  .478}1.25e0 & 2.25e-1 & \cellcolor[rgb]{ 1,  .922,  .518}1.13e0 & 8.25e-2 & \cellcolor[rgb]{ .973,  .412,  .42}1.43e0 & 3.65e-1 & \cellcolor[rgb]{ .722,  .839,  .498}1.03e0 & 1.33e-1 \\
          & 500   & \cellcolor[rgb]{ .388,  .745,  .482}9.18e-1 & 8.33e-2 & \cellcolor[rgb]{ .984,  .6,  .459}1.44e0 & 3.45e-1 & \cellcolor[rgb]{ 1,  .922,  .518}1.25e0 & 1.37e-1 & \cellcolor[rgb]{ .973,  .412,  .42}1.55e0 & 4.45e-1 & \cellcolor[rgb]{ .98,  .914,  .514}1.24e0 & 2.37e-1 \\
    \bottomrule
    \bottomrule
	\multicolumn{12}{m{0.5\textwidth}}{\tiny For each test case, $\text{FE} = D \times 10^{4}$. Results averaged from $20$ indepedent runs on each test case.}
    \end{tabular}%
  \label{tab:difficult}%
\end{table*}%

\par
It can be observed that \algoabbr{} validates its own effectiveness by achieving the best performance with an overwhelming margin.
\subsection{Stability Comparison using Confidence Interval Integration}
Stability is of great importance for optimization algorithms. We analyze the stability of \algoabbr{} by comparing it with those of the state-of-the-art algorithms. The evaluation criterion for stability is the scaled approximated size of area of the logarithmically-scaled confidence intervals. Let $\tau = \langle \bm{m}, \bm{b}_l, \bm{b}_u \rangle$ be the sampled statistical information (at time $t \in {1, \dots, T}$) of the IGD of a certain algorithm on a certain problem, where $\bm{m}$ is the mean IGD of independent runs at the sample times as well as $\bm{b}_l$ and $\bm{b}_u$ be the lowerbound and upperbound of the confidence interval calculated from the IGD values of the independent runs at the sample times, the stability criterion $V(\tau)$ is computed as:
$$V(\tau) = \sum_{i = 1}^{T}{\log{b_u^{(i)}} - \log{b_l^{(i)}}}$$
Note that the criterion requires the sample of points are the same, which we set to $101$ points per test case. The smaller the criterion, the better the stability (less variance). The results of such criterion on all test cases is provided in Tab. \ref{tab:stability}.
\begin{table*}[htbp]
\setlength{\tabcolsep}{1.5pt}
\scriptsize
  \centering
  \caption{Confidence Interval Integration Results for Stability}
    \begin{tabular}{cccccccccc}
    \toprule
    \toprule
    Problem & $M$     & \algoabbr{}  & \tiny CVEA3 & \tiny BCE-IBEA & \tiny fastCAR & \tiny CLIA  & \tiny AR-MOEA & \tiny A-NSGA-III & \tiny RVEA* \\
    \midrule
    \multirow{3}[1]{*}{F1} & 5     & \cellcolor[rgb]{ .98,  .529,  .443}7.66e+1 & \cellcolor[rgb]{ .435,  .757,  .482}8.48e0 & \cellcolor[rgb]{ 1,  .875,  .51}3.54e+1 & \cellcolor[rgb]{ .388,  .745,  .482}6.61e0 & \cellcolor[rgb]{ .98,  .529,  .443}7.65e+1 & \cellcolor[rgb]{ .839,  .875,  .506}2.36e+1 & \cellcolor[rgb]{ .608,  .808,  .494}1.48e+1 & \cellcolor[rgb]{ .973,  .412,  .42}9.06e+1 \\
          & 10    & \cellcolor[rgb]{ .973,  .412,  .42}8.51e+1 & \cellcolor[rgb]{ .996,  .8,  .494}5.76e+1 & \cellcolor[rgb]{ .976,  .482,  .435}8.02e+1 & \cellcolor[rgb]{ .992,  .741,  .486}6.17e+1 & \cellcolor[rgb]{ .875,  .882,  .51}3.98e+1 & \cellcolor[rgb]{ .467,  .765,  .486}9.97e0 & \cellcolor[rgb]{ .388,  .745,  .482}4.25e0 & \cellcolor[rgb]{ .839,  .875,  .506}3.73e+1 \\
          & 15    & \cellcolor[rgb]{ .976,  .424,  .424}8.58e+1 & \cellcolor[rgb]{ 1,  .871,  .51}5.37e+1 & \cellcolor[rgb]{ .933,  .902,  .514}4.56e+1 & \cellcolor[rgb]{ .973,  .412,  .42}8.65e+1 & \cellcolor[rgb]{ .835,  .875,  .506}3.92e+1 & \cellcolor[rgb]{ .388,  .745,  .482}9.31e0 & \cellcolor[rgb]{ .988,  .635,  .463}7.05e+1 & \cellcolor[rgb]{ .941,  .902,  .514}4.62e+1 \\
    \multirow{3}[0]{*}{F2} & 5     & \cellcolor[rgb]{ .973,  .412,  .42}4.20e+1 & \cellcolor[rgb]{ 1,  .863,  .51}2.61e+1 & \cellcolor[rgb]{ .514,  .78,  .486}5.85e0 & \cellcolor[rgb]{ 1,  .918,  .518}2.41e+1 & \cellcolor[rgb]{ .922,  .898,  .51}2.10e+1 & \cellcolor[rgb]{ .388,  .745,  .482}1.06e0 & \cellcolor[rgb]{ .992,  .918,  .514}2.37e+1 & \cellcolor[rgb]{ 1,  .855,  .506}2.64e+1 \\
          & 10    & \cellcolor[rgb]{ 1,  .851,  .506}4.45e+1 & \cellcolor[rgb]{ .988,  .682,  .475}5.72e+1 & \cellcolor[rgb]{ .988,  .667,  .471}5.83e+1 & \cellcolor[rgb]{ .78,  .855,  .502}2.75e+1 & \cellcolor[rgb]{ .612,  .808,  .494}1.87e+1 & \cellcolor[rgb]{ .973,  .412,  .42}7.71e+1 & \cellcolor[rgb]{ .898,  .89,  .51}3.38e+1 & \cellcolor[rgb]{ .388,  .745,  .482}6.65e0 \\
          & 15    & \cellcolor[rgb]{ .988,  .663,  .471}5.47e+1 & \cellcolor[rgb]{ .98,  .914,  .514}3.01e+1 & \cellcolor[rgb]{ .973,  .412,  .42}7.72e+1 & \cellcolor[rgb]{ .51,  .78,  .486}6.79e0 & \cellcolor[rgb]{ 1,  .902,  .518}3.28e+1 & \cellcolor[rgb]{ .388,  .745,  .482}6.95e-1 & \cellcolor[rgb]{ .584,  .8,  .49}1.05e+1 & \cellcolor[rgb]{ 1,  .914,  .518}3.20e+1 \\
    \multirow{3}[0]{*}{F3} & 5     & \cellcolor[rgb]{ .749,  .847,  .502}2.44e+1 & \cellcolor[rgb]{ .988,  .643,  .467}7.18e+1 & \cellcolor[rgb]{ .996,  .792,  .494}5.41e+1 & \cellcolor[rgb]{ .973,  .412,  .42}9.99e+1 & \cellcolor[rgb]{ .388,  .745,  .482}4.58e0 & \cellcolor[rgb]{ .678,  .827,  .498}2.05e+1 & \cellcolor[rgb]{ .925,  .898,  .51}3.40e+1 & \cellcolor[rgb]{ 1,  .89,  .514}4.19e+1 \\
          & 10    & \cellcolor[rgb]{ .937,  .902,  .514}3.28e+1 & \cellcolor[rgb]{ .388,  .745,  .482}6.34e0 & \cellcolor[rgb]{ 1,  .867,  .51}4.30e+1 & \cellcolor[rgb]{ .973,  .412,  .42}9.98e+1 & \cellcolor[rgb]{ 1,  .898,  .514}3.88e+1 & \cellcolor[rgb]{ .396,  .745,  .482}6.84e0 & \cellcolor[rgb]{ .8,  .863,  .506}2.62e+1 & \cellcolor[rgb]{ .98,  .533,  .443}8.45e+1 \\
          & 15    & \cellcolor[rgb]{ .388,  .745,  .482}8.71e0 & \cellcolor[rgb]{ .898,  .89,  .51}5.34e+1 & \cellcolor[rgb]{ .392,  .745,  .482}9.12e0 & \cellcolor[rgb]{ .973,  .412,  .42}9.47e+1 & \cellcolor[rgb]{ .992,  .745,  .486}7.35e+1 & \cellcolor[rgb]{ .486,  .773,  .486}1.76e+1 & \cellcolor[rgb]{ .996,  .788,  .494}7.07e+1 & \cellcolor[rgb]{ .988,  .655,  .467}7.91e+1 \\
    \multirow{3}[0]{*}{F4} & 5     & \cellcolor[rgb]{ .878,  .886,  .51}4.90e+1 & \cellcolor[rgb]{ .667,  .824,  .498}2.84e+1 & \cellcolor[rgb]{ .388,  .745,  .482}1.26e0 & \cellcolor[rgb]{ .973,  .412,  .42}9.31e+1 & \cellcolor[rgb]{ .992,  .737,  .482}7.24e+1 & \cellcolor[rgb]{ .976,  .459,  .431}9.01e+1 & \cellcolor[rgb]{ .518,  .78,  .486}1.40e+1 & \cellcolor[rgb]{ .976,  .439,  .427}9.14e+1 \\
          & 10    & \cellcolor[rgb]{ .533,  .784,  .49}3.08e+1 & \cellcolor[rgb]{ .388,  .745,  .482}2.74e+1 & \cellcolor[rgb]{ .788,  .859,  .502}3.66e+1 & \cellcolor[rgb]{ .996,  .843,  .506}4.63e+1 & \cellcolor[rgb]{ .545,  .788,  .49}3.10e+1 & \cellcolor[rgb]{ .976,  .482,  .435}6.85e+1 & \cellcolor[rgb]{ .98,  .506,  .439}6.69e+1 & \cellcolor[rgb]{ .973,  .412,  .42}7.26e+1 \\
          & 15    & \cellcolor[rgb]{ .608,  .808,  .494}7.42e0 & \cellcolor[rgb]{ .973,  .914,  .514}1.83e+1 & \cellcolor[rgb]{ .388,  .745,  .482}7.31e-1 & \cellcolor[rgb]{ .494,  .776,  .486}4.01e0 & \cellcolor[rgb]{ 1,  .914,  .518}1.99e+1 & \cellcolor[rgb]{ .973,  .412,  .42}6.37e+1 & \cellcolor[rgb]{ .984,  .58,  .455}4.90e+1 & \cellcolor[rgb]{ .992,  .706,  .478}3.82e+1 \\
    \multirow{3}[0]{*}{F5} & 5     & \cellcolor[rgb]{ .988,  .698,  .478}8.52e+1 & \cellcolor[rgb]{ .388,  .745,  .482}4.02e0 & \cellcolor[rgb]{ .769,  .855,  .502}5.19e+1 & \cellcolor[rgb]{ .98,  .506,  .439}8.92e+1 & \cellcolor[rgb]{ .961,  .91,  .514}7.60e+1 & \cellcolor[rgb]{ .973,  .412,  .42}9.12e+1 & \cellcolor[rgb]{ .882,  .886,  .51}6.62e+1 & \cellcolor[rgb]{ .98,  .502,  .439}8.93e+1 \\
          & 10    & \cellcolor[rgb]{ .541,  .788,  .49}7.61e0 & \cellcolor[rgb]{ 1,  .898,  .514}3.00e+1 & \cellcolor[rgb]{ .918,  .894,  .51}2.32e+1 & \cellcolor[rgb]{ .973,  .412,  .42}9.05e+1 & \cellcolor[rgb]{ .388,  .745,  .482}1.10e0 & \cellcolor[rgb]{ .992,  .714,  .478}5.29e+1 & \cellcolor[rgb]{ .988,  .675,  .471}5.80e+1 & \cellcolor[rgb]{ .694,  .831,  .498}1.40e+1 \\
          & 15    & \cellcolor[rgb]{ .745,  .847,  .502}3.97e+1 & \cellcolor[rgb]{ .388,  .745,  .482}5.21e0 & \cellcolor[rgb]{ .996,  .843,  .506}6.87e+1 & \cellcolor[rgb]{ .973,  .412,  .42}9.41e+1 & \cellcolor[rgb]{ .651,  .82,  .494}3.04e+1 & \cellcolor[rgb]{ .949,  .906,  .514}5.91e+1 & \cellcolor[rgb]{ .992,  .714,  .478}7.63e+1 & \cellcolor[rgb]{ .988,  .675,  .471}7.87e+1 \\
    \multirow{3}[0]{*}{F6} & 5     & \cellcolor[rgb]{ .502,  .776,  .486}1.42e+1 & \cellcolor[rgb]{ .973,  .412,  .42}7.22e+1 & \cellcolor[rgb]{ .976,  .424,  .424}7.19e+1 & \cellcolor[rgb]{ .455,  .765,  .486}1.04e+1 & \cellcolor[rgb]{ .957,  .91,  .514}5.13e+1 & \cellcolor[rgb]{ .976,  .443,  .427}7.12e+1 & \cellcolor[rgb]{ .388,  .745,  .482}4.92e0 & \cellcolor[rgb]{ .996,  .831,  .502}5.76e+1 \\
          & 10    & \cellcolor[rgb]{ .988,  .69,  .475}7.82e+1 & \cellcolor[rgb]{ 1,  .902,  .514}6.13e+1 & \cellcolor[rgb]{ .925,  .898,  .51}5.35e+1 & \cellcolor[rgb]{ .973,  .412,  .42}1.00e2 & \cellcolor[rgb]{ .388,  .745,  .482}9.62e0 & \cellcolor[rgb]{ .447,  .761,  .482}1.47e+1 & \cellcolor[rgb]{ 1,  .894,  .514}6.20e+1 & \cellcolor[rgb]{ .976,  .914,  .514}5.79e+1 \\
          & 15    & \cellcolor[rgb]{ .98,  .557,  .451}7.83e+1 & \cellcolor[rgb]{ .98,  .51,  .439}8.18e+1 & \cellcolor[rgb]{ .98,  .494,  .435}8.30e+1 & \cellcolor[rgb]{ .973,  .412,  .42}8.90e+1 & \cellcolor[rgb]{ .502,  .776,  .486}1.73e+1 & \cellcolor[rgb]{ .388,  .745,  .482}9.30e0 & \cellcolor[rgb]{ .431,  .757,  .482}1.23e+1 & \cellcolor[rgb]{ .604,  .804,  .494}2.41e+1 \\
    \multirow{3}[0]{*}{F7} & 5     & \cellcolor[rgb]{ .949,  .906,  .514}4.63e+1 & \cellcolor[rgb]{ .388,  .745,  .482}1.54e+1 & \cellcolor[rgb]{ .984,  .631,  .463}6.54e+1 & \cellcolor[rgb]{ .843,  .875,  .506}4.04e+1 & \cellcolor[rgb]{ .976,  .459,  .431}7.49e+1 & \cellcolor[rgb]{ .973,  .412,  .42}7.75e+1 & \cellcolor[rgb]{ .498,  .776,  .486}2.16e+1 & \cellcolor[rgb]{ 1,  .878,  .51}5.15e+1 \\
          & 10    & \cellcolor[rgb]{ .388,  .745,  .482}1.23e0 & \cellcolor[rgb]{ .937,  .902,  .514}4.44e+1 & \cellcolor[rgb]{ .992,  .773,  .49}6.19e+1 & \cellcolor[rgb]{ .631,  .812,  .494}2.04e+1 & \cellcolor[rgb]{ 1,  .867,  .51}5.43e+1 & \cellcolor[rgb]{ .984,  .565,  .451}7.95e+1 & \cellcolor[rgb]{ .973,  .412,  .42}9.23e+1 & \cellcolor[rgb]{ .537,  .788,  .49}1.31e+1 \\
          & 15    & \cellcolor[rgb]{ .988,  .918,  .514}5.34e+1 & \cellcolor[rgb]{ .992,  .725,  .482}6.39e+1 & \cellcolor[rgb]{ .471,  .769,  .486}2.42e+1 & \cellcolor[rgb]{ .455,  .761,  .482}2.31e+1 & \cellcolor[rgb]{ .388,  .745,  .482}1.93e+1 & \cellcolor[rgb]{ .973,  .412,  .42}7.98e+1 & \cellcolor[rgb]{ 1,  .914,  .518}5.43e+1 & \cellcolor[rgb]{ 1,  .871,  .51}5.65e+1 \\
    \multirow{3}[0]{*}{F8} & 5     & \cellcolor[rgb]{ .957,  .906,  .514}5.29e+1 & \cellcolor[rgb]{ .973,  .412,  .42}7.62e+1 & \cellcolor[rgb]{ .98,  .541,  .447}7.12e+1 & \cellcolor[rgb]{ .388,  .745,  .482}7.84e0 & \cellcolor[rgb]{ .804,  .863,  .506}4.07e+1 & \cellcolor[rgb]{ .996,  .843,  .502}5.94e+1 & \cellcolor[rgb]{ .592,  .804,  .494}2.41e+1 & \cellcolor[rgb]{ .98,  .514,  .439}7.23e+1 \\
          & 10    & \cellcolor[rgb]{ .976,  .914,  .514}6.51e+1 & \cellcolor[rgb]{ .973,  .412,  .42}8.74e+1 & \cellcolor[rgb]{ .992,  .737,  .482}7.43e+1 & \cellcolor[rgb]{ .988,  .702,  .478}7.57e+1 & \cellcolor[rgb]{ .847,  .878,  .506}5.57e+1 & \cellcolor[rgb]{ .863,  .882,  .51}5.67e+1 & \cellcolor[rgb]{ 1,  .882,  .514}6.84e+1 & \cellcolor[rgb]{ .388,  .745,  .482}2.13e+1 \\
          & 15    & \cellcolor[rgb]{ .58,  .8,  .49}2.28e+1 & \cellcolor[rgb]{ .388,  .745,  .482}2.14e0 & \cellcolor[rgb]{ .988,  .651,  .467}7.96e+1 & \cellcolor[rgb]{ .957,  .906,  .514}6.27e+1 & \cellcolor[rgb]{ .973,  .412,  .42}9.02e+1 & \cellcolor[rgb]{ .996,  .824,  .502}7.18e+1 & \cellcolor[rgb]{ .992,  .733,  .482}7.58e+1 & \cellcolor[rgb]{ .784,  .859,  .502}4.45e+1 \\
    \multirow{3}[0]{*}{F9} & 5     & \cellcolor[rgb]{ .984,  .627,  .463}7.72e+1 & \cellcolor[rgb]{ .533,  .784,  .49}1.75e+1 & \cellcolor[rgb]{ .973,  .412,  .42}9.07e+1 & \cellcolor[rgb]{ .835,  .875,  .506}4.41e+1 & \cellcolor[rgb]{ .992,  .718,  .478}7.15e+1 & \cellcolor[rgb]{ .984,  .616,  .459}7.80e+1 & \cellcolor[rgb]{ .388,  .745,  .482}4.29e0 & \cellcolor[rgb]{ .855,  .878,  .506}4.57e+1 \\
          & 10    & \cellcolor[rgb]{ .388,  .745,  .482}9.62e0 & \cellcolor[rgb]{ .98,  .498,  .439}7.49e+1 & \cellcolor[rgb]{ .996,  .788,  .494}5.32e+1 & \cellcolor[rgb]{ .439,  .757,  .482}1.25e+1 & \cellcolor[rgb]{ .976,  .486,  .435}7.56e+1 & \cellcolor[rgb]{ .973,  .412,  .42}8.10e+1 & \cellcolor[rgb]{ .816,  .867,  .506}3.33e+1 & \cellcolor[rgb]{ .788,  .859,  .502}3.16e+1 \\
          & 15    & \cellcolor[rgb]{ .98,  .533,  .443}6.77e+1 & \cellcolor[rgb]{ .388,  .745,  .482}7.34e0 & \cellcolor[rgb]{ .973,  .412,  .42}7.32e+1 & \cellcolor[rgb]{ .89,  .89,  .51}4.22e+1 & \cellcolor[rgb]{ .506,  .776,  .486}1.56e+1 & \cellcolor[rgb]{ .98,  .494,  .435}6.95e+1 & \cellcolor[rgb]{ .992,  .765,  .49}5.69e+1 & \cellcolor[rgb]{ .667,  .824,  .498}2.66e+1 \\
    \multirow{3}[0]{*}{F10} & 5     & \cellcolor[rgb]{ .604,  .804,  .494}1.48e+1 & \cellcolor[rgb]{ .718,  .839,  .498}2.10e+1 & \cellcolor[rgb]{ .741,  .847,  .502}2.23e+1 & \cellcolor[rgb]{ .973,  .412,  .42}7.93e+1 & \cellcolor[rgb]{ .976,  .447,  .427}7.66e+1 & \cellcolor[rgb]{ .388,  .745,  .482}3.40e0 & \cellcolor[rgb]{ .992,  .765,  .49}4.94e+1 & \cellcolor[rgb]{ .98,  .49,  .435}7.28e+1 \\
          & 10    & \cellcolor[rgb]{ .702,  .835,  .498}4.39e+1 & \cellcolor[rgb]{ 1,  .906,  .518}6.24e+1 & \cellcolor[rgb]{ .388,  .745,  .482}2.53e+1 & \cellcolor[rgb]{ .973,  .412,  .42}9.00e+1 & \cellcolor[rgb]{ .871,  .882,  .51}5.39e+1 & \cellcolor[rgb]{ .988,  .682,  .475}7.50e+1 & \cellcolor[rgb]{ .98,  .914,  .514}6.06e+1 & \cellcolor[rgb]{ .98,  .529,  .443}8.35e+1 \\
          & 15    & \cellcolor[rgb]{ .82,  .867,  .506}5.32e+1 & \cellcolor[rgb]{ .388,  .745,  .482}1.80e0 & \cellcolor[rgb]{ .878,  .886,  .51}6.00e+1 & \cellcolor[rgb]{ .973,  .412,  .42}8.44e+1 & \cellcolor[rgb]{ 1,  .878,  .51}7.52e+1 & \cellcolor[rgb]{ .996,  .788,  .494}7.70e+1 & \cellcolor[rgb]{ .992,  .918,  .514}7.33e+1 & \cellcolor[rgb]{ .988,  .682,  .475}7.90e+1 \\
    \multirow{3}[0]{*}{F11} & 5     & \cellcolor[rgb]{ .553,  .792,  .49}2.58e+1 & \cellcolor[rgb]{ .988,  .651,  .467}7.88e+1 & \cellcolor[rgb]{ .973,  .412,  .42}8.80e+1 & \cellcolor[rgb]{ .984,  .569,  .451}8.19e+1 & \cellcolor[rgb]{ .388,  .745,  .482}1.01e+1 & \cellcolor[rgb]{ .494,  .773,  .486}2.04e+1 & \cellcolor[rgb]{ .886,  .886,  .51}5.74e+1 & \cellcolor[rgb]{ .98,  .522,  .443}8.38e+1 \\
          & 10    & \cellcolor[rgb]{ .788,  .859,  .502}3.32e+1 & \cellcolor[rgb]{ .988,  .667,  .471}6.84e+1 & \cellcolor[rgb]{ .702,  .835,  .498}2.81e+1 & \cellcolor[rgb]{ .973,  .412,  .42}9.06e+1 & \cellcolor[rgb]{ .388,  .745,  .482}9.57e0 & \cellcolor[rgb]{ .996,  .784,  .494}5.79e+1 & \cellcolor[rgb]{ .98,  .557,  .447}7.81e+1 & \cellcolor[rgb]{ .733,  .843,  .502}3.01e+1 \\
          & 15    & \cellcolor[rgb]{ .945,  .906,  .514}5.18e+1 & \cellcolor[rgb]{ 1,  .878,  .51}5.81e+1 & \cellcolor[rgb]{ 1,  .871,  .51}5.87e+1 & \cellcolor[rgb]{ .973,  .412,  .42}9.04e+1 & \cellcolor[rgb]{ .729,  .843,  .502}3.83e+1 & \cellcolor[rgb]{ .984,  .58,  .455}7.89e+1 & \cellcolor[rgb]{ .388,  .745,  .482}1.73e+1 & \cellcolor[rgb]{ .639,  .816,  .494}3.30e+1 \\
    \multirow{3}[0]{*}{F12} & 5     & \cellcolor[rgb]{ .631,  .816,  .494}5.18e+1 & \cellcolor[rgb]{ .98,  .522,  .443}9.13e+1 & \cellcolor[rgb]{ 1,  .878,  .51}8.10e+1 & \cellcolor[rgb]{ .996,  .808,  .498}8.30e+1 & \cellcolor[rgb]{ .388,  .745,  .482}3.31e+1 & \cellcolor[rgb]{ .965,  .91,  .514}7.71e+1 & \cellcolor[rgb]{ .98,  .914,  .514}7.84e+1 & \cellcolor[rgb]{ .973,  .412,  .42}9.43e+1 \\
          & 10    & \cellcolor[rgb]{ .737,  .843,  .502}2.49e+1 & \cellcolor[rgb]{ .973,  .914,  .514}3.15e+1 & \cellcolor[rgb]{ .973,  .412,  .42}8.68e+1 & \cellcolor[rgb]{ .976,  .42,  .424}8.62e+1 & \cellcolor[rgb]{ .933,  .902,  .514}3.04e+1 & \cellcolor[rgb]{ .388,  .745,  .482}1.50e+1 & \cellcolor[rgb]{ 1,  .918,  .518}3.29e+1 & \cellcolor[rgb]{ .988,  .694,  .475}5.70e+1 \\
          & 15    & \cellcolor[rgb]{ .973,  .412,  .42}9.06e+1 & \cellcolor[rgb]{ .984,  .584,  .455}7.20e+1 & \cellcolor[rgb]{ .388,  .745,  .482}5.79e0 & \cellcolor[rgb]{ .98,  .529,  .443}7.80e+1 & \cellcolor[rgb]{ .624,  .812,  .494}1.69e+1 & \cellcolor[rgb]{ .714,  .839,  .498}2.10e+1 & \cellcolor[rgb]{ .996,  .804,  .498}4.74e+1 & \cellcolor[rgb]{ .6,  .804,  .494}1.58e+1 \\
    \multirow{3}[0]{*}{F13} & 5     & \cellcolor[rgb]{ .467,  .765,  .486}9.41e0 & \cellcolor[rgb]{ .467,  .765,  .486}9.46e0 & \cellcolor[rgb]{ .988,  .678,  .475}5.91e+1 & \cellcolor[rgb]{ 1,  .871,  .51}3.51e+1 & \cellcolor[rgb]{ .388,  .745,  .482}6.57e0 & \cellcolor[rgb]{ .98,  .529,  .443}7.79e+1 & \cellcolor[rgb]{ .816,  .867,  .506}2.20e+1 & \cellcolor[rgb]{ .973,  .412,  .42}9.25e+1 \\
          & 10    & \cellcolor[rgb]{ .965,  .91,  .514}6.34e+1 & \cellcolor[rgb]{ .996,  .784,  .494}7.22e+1 & \cellcolor[rgb]{ .973,  .412,  .42}8.76e+1 & \cellcolor[rgb]{ .435,  .757,  .482}1.06e+1 & \cellcolor[rgb]{ .388,  .745,  .482}5.89e0 & \cellcolor[rgb]{ .996,  .847,  .506}6.97e+1 & \cellcolor[rgb]{ .608,  .808,  .494}2.79e+1 & \cellcolor[rgb]{ .976,  .475,  .431}8.51e+1 \\
          & 15    & \cellcolor[rgb]{ 1,  .863,  .51}7.31e+1 & \cellcolor[rgb]{ .388,  .745,  .482}6.98e0 & \cellcolor[rgb]{ .988,  .694,  .475}7.69e+1 & \cellcolor[rgb]{ .984,  .918,  .514}7.05e+1 & \cellcolor[rgb]{ .773,  .855,  .502}4.81e+1 & \cellcolor[rgb]{ .973,  .412,  .42}8.30e+1 & \cellcolor[rgb]{ .718,  .839,  .498}4.21e+1 & \cellcolor[rgb]{ .992,  .733,  .482}7.60e+1 \\
    \multirow{3}[0]{*}{F14} & 5     & \cellcolor[rgb]{ .545,  .788,  .49}3.79e+1 & \cellcolor[rgb]{ .992,  .761,  .486}6.72e+1 & \cellcolor[rgb]{ .976,  .431,  .424}8.25e+1 & \cellcolor[rgb]{ .996,  .843,  .506}6.34e+1 & \cellcolor[rgb]{ .388,  .745,  .482}3.04e+1 & \cellcolor[rgb]{ .918,  .898,  .51}5.58e+1 & \cellcolor[rgb]{ .506,  .776,  .486}3.60e+1 & \cellcolor[rgb]{ .973,  .412,  .42}8.34e+1 \\
          & 10    & \cellcolor[rgb]{ .973,  .412,  .42}8.89e+1 & \cellcolor[rgb]{ .51,  .78,  .486}2.15e+1 & \cellcolor[rgb]{ .804,  .863,  .506}4.91e+1 & \cellcolor[rgb]{ .996,  .918,  .514}6.71e+1 & \cellcolor[rgb]{ 1,  .875,  .51}6.94e+1 & \cellcolor[rgb]{ .98,  .498,  .439}8.53e+1 & \cellcolor[rgb]{ .388,  .745,  .482}9.88e0 & \cellcolor[rgb]{ 1,  .922,  .518}6.74e+1 \\
          & 15    & \cellcolor[rgb]{ .773,  .855,  .502}1.51e+1 & \cellcolor[rgb]{ .741,  .843,  .502}1.47e+1 & \cellcolor[rgb]{ 1,  .902,  .514}2.00e+1 & \cellcolor[rgb]{ .98,  .518,  .443}6.38e+1 & \cellcolor[rgb]{ .388,  .745,  .482}1.08e+1 & \cellcolor[rgb]{ .973,  .412,  .42}7.57e+1 & \cellcolor[rgb]{ .996,  .78,  .494}3.37e+1 & \cellcolor[rgb]{ .392,  .745,  .482}1.09e+1 \\
    \multirow{3}[1]{*}{F15} & 5     & \cellcolor[rgb]{ .941,  .902,  .514}5.78e+1 & \cellcolor[rgb]{ .992,  .722,  .482}7.15e+1 & \cellcolor[rgb]{ .976,  .914,  .514}5.93e+1 & \cellcolor[rgb]{ 1,  .863,  .506}6.37e+1 & \cellcolor[rgb]{ .882,  .886,  .51}5.52e+1 & \cellcolor[rgb]{ .388,  .745,  .482}3.36e+1 & \cellcolor[rgb]{ .973,  .412,  .42}8.84e+1 & \cellcolor[rgb]{ 1,  .906,  .518}6.11e+1 \\
          & 10    & \cellcolor[rgb]{ .973,  .412,  .42}9.75e+1 & \cellcolor[rgb]{ .604,  .804,  .494}2.78e+1 & \cellcolor[rgb]{ .388,  .745,  .482}1.09e+1 & \cellcolor[rgb]{ .957,  .91,  .514}5.57e+1 & \cellcolor[rgb]{ 1,  .882,  .51}6.19e+1 & \cellcolor[rgb]{ .996,  .835,  .502}6.54e+1 & \cellcolor[rgb]{ .796,  .863,  .506}4.29e+1 & \cellcolor[rgb]{ .992,  .714,  .478}7.47e+1 \\
          & 15    & \cellcolor[rgb]{ .388,  .745,  .482}1.14e+1 & \cellcolor[rgb]{ 1,  .918,  .518}5.97e+1 & \cellcolor[rgb]{ .973,  .412,  .42}8.79e+1 & \cellcolor[rgb]{ .992,  .737,  .482}6.98e+1 & \cellcolor[rgb]{ .961,  .91,  .514}5.66e+1 & \cellcolor[rgb]{ .976,  .443,  .427}8.63e+1 & \cellcolor[rgb]{ .565,  .796,  .49}2.55e+1 & \cellcolor[rgb]{ .996,  .918,  .514}5.92e+1 \\
    \midrule
    \multicolumn{2}{c}{Friedman} & \cellcolor[rgb]{ .729,  .804,  .906}4.07 & \cellcolor[rgb]{ .804,  .859,  .933}4.20 & \cellcolor[rgb]{ .984,  .714,  .722}4.93 & \cellcolor[rgb]{ .973,  .412,  .42}5.38 & \cellcolor[rgb]{ .353,  .541,  .776}3.40 & \cellcolor[rgb]{ .984,  .773,  .784}4.84 & \cellcolor[rgb]{ .702,  .788,  .898}4.02 & \cellcolor[rgb]{ .98,  .565,  .573}5.16 \\
    \bottomrule
    \bottomrule
	\multicolumn{10}{m{0.4\textwidth}}{For Friedman test, $\alpha = 0.05$, $p = 4.20 \times 10^{-3}$ and $\chi^2 = 20.74$.} \\
    \end{tabular}%
  \label{tab:stability}%
\end{table*}%

\par
The Friedman test shows that \algoabbr{} has modest stability among all compared algorithms, showing that the achieved competitive performance in the experiments is in a degree reliable.
\subsection{Component Validation}
\subsubsection{Effectiveness of Individual Archive}
We demonstrate the effectiveness of IA by comparing the performance \algoabbr{} has achieved to the version of \algoabbr{} without IA (using the population to adjust references instead of IA). The results are given in Tab. \ref{tab:disableIA}.
\begin{table*}[htbp]
\setlength{\tabcolsep}{1.5pt}
\scriptsize
  \centering
  \caption{Comparison for \algoabbr{} with IA and \algoabbr{} without IA}
    \begin{tabular}{cccccc}
    \toprule
    \toprule
    \multirow{2}[2]{*}{Problem} & \multirow{2}[2]{*}{$M$} & \multicolumn{2}{c}{\textbf{\algoabbr{}}} & \multicolumn{2}{c}{\textbf{\algoabbr{} w/o IA}} \\
          &       & Mean  & Std   & Mean  & Std \\
    \midrule
    \multirow{3}[1]{*}{F1} & 5     & \cellcolor[rgb]{ .353,  .541,  .776}1.08e-1 & 7.25e-4 & \cellcolor[rgb]{ .973,  .412,  .42}1.16e-1 & 7.84e-4 \\
          & 10    & \cellcolor[rgb]{ .353,  .541,  .776}2.33e-1 & 7.45e-3 & \cellcolor[rgb]{ .973,  .412,  .42}2.50e-1 & 7.94e-3 \\
          & 15    & \cellcolor[rgb]{ .353,  .541,  .776}2.67e-1 & 8.85e-3 & \cellcolor[rgb]{ .973,  .412,  .42}2.85e-1 & 8.96e-3 \\
    \multirow{3}[0]{*}{F2} & 5     & \cellcolor[rgb]{ .353,  .541,  .776}9.65e-2 & 2.85e-3 & \cellcolor[rgb]{ .973,  .412,  .42}1.00e-1 & 2.89e-3 \\
          & 10    & \cellcolor[rgb]{ .973,  .412,  .42}1.53e-1 & 2.99e-3 & \cellcolor[rgb]{ .353,  .541,  .776}1.50e-1 & 3.01e-3 \\
          & 15    & \cellcolor[rgb]{ .353,  .541,  .776}1.64e-1 & 4.35e-3 & \cellcolor[rgb]{ .973,  .412,  .42}1.67e-1 & 4.53e-3 \\
    \multirow{3}[0]{*}{F3} & 5     & \cellcolor[rgb]{ .973,  .412,  .42}6.82e-2 & 1.90e-3 & \cellcolor[rgb]{ .353,  .541,  .776}6.76e-2 & 1.94e-3 \\
          & 10    & \cellcolor[rgb]{ .353,  .541,  .776}8.36e-2 & 2.20e-3 & \cellcolor[rgb]{ .973,  .412,  .42}8.68e-2 & 2.35e-3 \\
          & 15    & \cellcolor[rgb]{ .973,  .412,  .42}9.11e-2 & 1.27e-3 & \cellcolor[rgb]{ .353,  .541,  .776}8.93e-2 & 1.38e-3 \\
    \multirow{3}[0]{*}{F4} & 5     & \cellcolor[rgb]{ .353,  .541,  .776}1.77e0 & 3.25e-2 & \cellcolor[rgb]{ .973,  .412,  .42}1.78e0 & 3.49e-2 \\
          & 10    & \cellcolor[rgb]{ .353,  .541,  .776}7.50e1 & 2.25e0 & \cellcolor[rgb]{ .973,  .412,  .42}7.60e1 & 2.39e0 \\
          & 15    & \cellcolor[rgb]{ .353,  .541,  .776}2.80e3 & 1.23e2 & \cellcolor[rgb]{ .973,  .412,  .42}2.95e3 & 1.28e2 \\
    \multirow{3}[0]{*}{F5} & 5     & \cellcolor[rgb]{ .353,  .541,  .776}1.96e0 & 9.55e-3 & \cellcolor[rgb]{ .973,  .412,  .42}1.98e0 & 9.93e-3 \\
          & 10    & \cellcolor[rgb]{ .973,  .412,  .42}8.82e1 & 1.19e0 & \cellcolor[rgb]{ .353,  .541,  .776}8.82e1 & 1.21e0 \\
          & 15    & \cellcolor[rgb]{ .353,  .541,  .776}2.35e3 & 3.07e2 & \cellcolor[rgb]{ .973,  .412,  .42}2.51e3 & 3.21e2 \\
    \multirow{3}[0]{*}{F6} & 5     & \cellcolor[rgb]{ .353,  .541,  .776}1.82e-3 & 1.43e-4 & \cellcolor[rgb]{ .973,  .412,  .42}1.95e-3 & 1.53e-4 \\
          & 10    & \cellcolor[rgb]{ .353,  .541,  .776}2.75e-3 & 6.25e-4 & \cellcolor[rgb]{ .973,  .412,  .42}2.92e-3 & 6.35e-4 \\
          & 15    & \cellcolor[rgb]{ .353,  .541,  .776}1.03e-2 & 9.25e-3 & \cellcolor[rgb]{ .973,  .412,  .42}1.08e-2 & 1.01e-2 \\
    \multirow{3}[0]{*}{F7} & 5     & \cellcolor[rgb]{ .353,  .541,  .776}2.78e-1 & 3.25e-3 & \cellcolor[rgb]{ .973,  .412,  .42}2.85e-1 & 3.50e-3 \\
          & 10    & \cellcolor[rgb]{ .973,  .412,  .42}8.35e-1 & 1.03e-1 & \cellcolor[rgb]{ .353,  .541,  .776}8.16e-1 & 1.06e-1 \\
          & 15    & \cellcolor[rgb]{ .353,  .541,  .776}1.80e0 & 3.40e-1 & \cellcolor[rgb]{ .973,  .412,  .42}1.81e0 & 3.67e-1 \\
    \multirow{3}[0]{*}{F8} & 5     & \cellcolor[rgb]{ .973,  .412,  .42}9.17e-2 & 4.32e-3 & \cellcolor[rgb]{ .353,  .541,  .776}9.15e-2 & 4.34e-3 \\
          & 10    & \cellcolor[rgb]{ .353,  .541,  .776}1.35e-1 & 4.99e-3 & \cellcolor[rgb]{ .973,  .412,  .42}1.39e-1 & 5.23e-3 \\
          & 15    & \cellcolor[rgb]{ .353,  .541,  .776}1.93e-1 & 1.87e-2 & \cellcolor[rgb]{ .973,  .412,  .42}2.07e-1 & 1.87e-2 \\
    \multirow{3}[0]{*}{F9} & 5     & \cellcolor[rgb]{ .353,  .541,  .776}9.17e-2 & 6.35e-3 & \cellcolor[rgb]{ .973,  .412,  .42}9.71e-2 & 6.87e-3 \\
          & 10    & \cellcolor[rgb]{ .973,  .412,  .42}1.85e-1 & 1.27e-2 & \cellcolor[rgb]{ .353,  .541,  .776}1.81e-1 & 1.37e-2 \\
          & 15    & \cellcolor[rgb]{ .353,  .541,  .776}2.25e-1 & 6.25e-2 & \cellcolor[rgb]{ .973,  .412,  .42}2.40e-1 & 6.56e-2 \\
    \multirow{3}[0]{*}{F10} & 5     & \cellcolor[rgb]{ .353,  .541,  .776}3.87e-1 & 1.45e-2 & \cellcolor[rgb]{ .973,  .412,  .42}3.96e-1 & 1.48e-2 \\
          & 10    & \cellcolor[rgb]{ .353,  .541,  .776}1.02e0 & 2.98e-2 & \cellcolor[rgb]{ .973,  .412,  .42}1.07e0 & 3.04e-2 \\
          & 15    & \cellcolor[rgb]{ .973,  .412,  .42}1.38e0 & 3.91e-2 & \cellcolor[rgb]{ .353,  .541,  .776}1.35e0 & 4.06e-2 \\
    \multirow{3}[0]{*}{F11} & 5     & \cellcolor[rgb]{ .353,  .541,  .776}3.89e-1 & 1.93e-3 & \cellcolor[rgb]{ .973,  .412,  .42}3.96e-1 & 1.96e-3 \\
          & 10    & \cellcolor[rgb]{ .353,  .541,  .776}1.15e0 & 2.26e-2 & \cellcolor[rgb]{ .973,  .412,  .42}1.16e0 & 2.29e-2 \\
          & 15    & \cellcolor[rgb]{ .353,  .541,  .776}1.43e0 & 3.71e-2 & \cellcolor[rgb]{ .973,  .412,  .42}1.52e0 & 3.75e-2 \\
    \multirow{3}[0]{*}{F12} & 5     & \cellcolor[rgb]{ .353,  .541,  .776}9.36e-1 & 3.43e-3 & \cellcolor[rgb]{ .973,  .412,  .42}9.69e-1 & 3.42e-3 \\
          & 10    & \cellcolor[rgb]{ .353,  .541,  .776}4.60e0 & 1.81e-2 & \cellcolor[rgb]{ .973,  .412,  .42}4.70e0 & 1.81e-2 \\
          & 15    & \cellcolor[rgb]{ .353,  .541,  .776}7.73e0 & 8.12e-2 & \cellcolor[rgb]{ .973,  .412,  .42}7.77e0 & 8.34e-2 \\
    \multirow{3}[0]{*}{F13} & 5     & \cellcolor[rgb]{ .353,  .541,  .776}8.83e-2 & 1.25e-2 & \cellcolor[rgb]{ .973,  .412,  .42}9.39e-2 & 1.27e-2 \\
          & 10    & \cellcolor[rgb]{ .973,  .412,  .42}1.12e-1 & 5.25e-2 & \cellcolor[rgb]{ .353,  .541,  .776}1.12e-1 & 5.33e-2 \\
          & 15    & \cellcolor[rgb]{ .973,  .412,  .42}1.43e-1 & 6.35e-2 & \cellcolor[rgb]{ .353,  .541,  .776}1.43e-1 & 6.39e-2 \\
    \multirow{3}[0]{*}{F14} & 5     & \cellcolor[rgb]{ .353,  .541,  .776}3.42e-1 & 2.47e-2 & \cellcolor[rgb]{ .973,  .412,  .42}3.55e-1 & 2.54e-2 \\
          & 10    & \cellcolor[rgb]{ .353,  .541,  .776}5.49e-1 & 7.73e-2 & \cellcolor[rgb]{ .973,  .412,  .42}5.79e-1 & 7.96e-2 \\
          & 15    & \cellcolor[rgb]{ .353,  .541,  .776}6.22e-1 & 1.43e-1 & \cellcolor[rgb]{ .973,  .412,  .42}6.52e-1 & 1.54e-1 \\
    \multirow{3}[1]{*}{F15} & 5     & \cellcolor[rgb]{ .353,  .541,  .776}3.06e-1 & 4.33e-2 & \cellcolor[rgb]{ .973,  .412,  .42}3.26e-1 & 4.48e-2 \\
          & 10    & \cellcolor[rgb]{ .973,  .412,  .42}8.99e-1 & 9.81e-2 & \cellcolor[rgb]{ .353,  .541,  .776}8.98e-1 & 9.82e-2 \\
          & 15    & \cellcolor[rgb]{ .353,  .541,  .776}1.02e0 & 1.35e-1 & \cellcolor[rgb]{ .973,  .412,  .42}1.04e0 & 1.46e-1 \\
    \midrule
    \multicolumn{2}{c}{$t$-test} & \multicolumn{2}{c}{~} & \multicolumn{2}{c}{16/22/7} \\
    \bottomrule
    \bottomrule
    \end{tabular}%
  \label{tab:disableIA}%
\end{table*}%

It can be observed that generally \algoabbr{} with IA performs better than \algoabbr{} without IA. IA can bring changes for performance on different types of problems: reason for performing better on problems with full FOS, IA prevents the adaptation engine from mistakenly initiating due to the stochastic perturbations of the evolutionary process; On problems with partial FOS, apart from the first reason, the additional is that IA keeps all the individuals that could activate the reference vectors, whereas in the population, due to the limit of the size of population, useful cluster centers are often lost.
\subsubsection{Effectiveness of Adaptation}
To validate the effectiveness of RA, we feed the adaptation engine with different types of simulated tracked PFs to check if the adaptation engine can effectively adapt the reference vectors to the distribution of the current PF. To visually verify the effectiveness, we have crafted $4$ different simulated test cases, where respectively, the current PF is segmented into parts by the imaginary infeasible parts of the corresponding objective spaces. The results are given in Fig. \ref{fig:markov}.
\begin{figure*}
\centering
\subfloat[$N = 24, \#active = 23$]{
\captionsetup{justification = centering}
\includegraphics[width=0.45\textwidth]{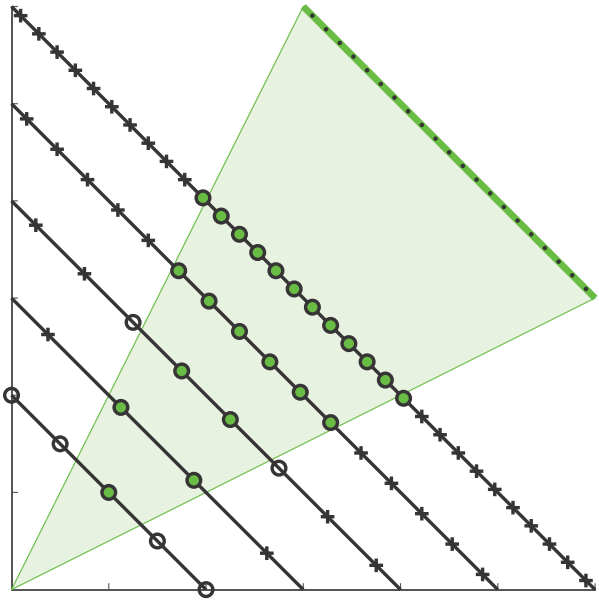}}
\hfill
\subfloat[$N = 24, \#active = 22$]{
\captionsetup{justification = centering}
\includegraphics[width=0.45\textwidth]{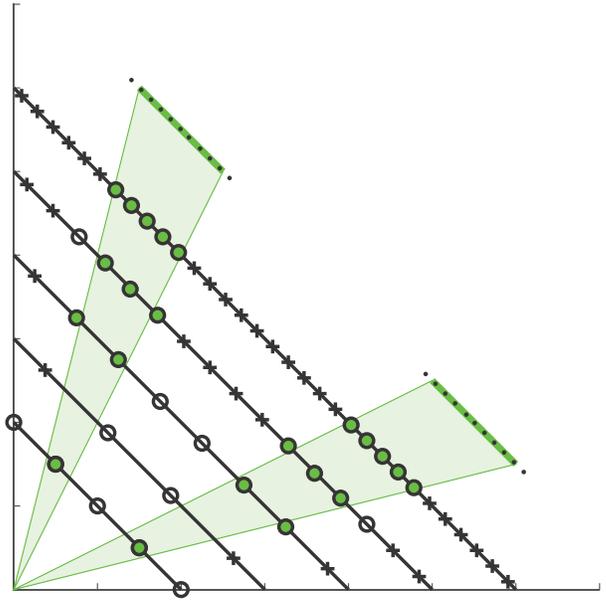}}

\subfloat[$N = 24, \#active = 19$]{
\captionsetup{justification = centering}
\includegraphics[width=0.45\textwidth]{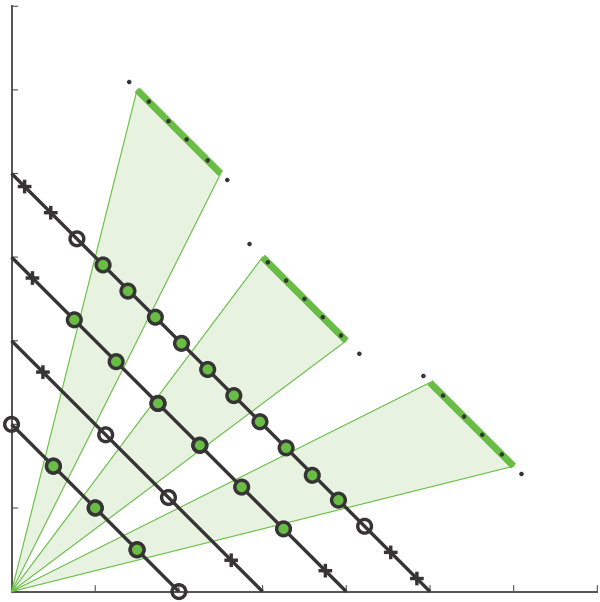}}
\hfill
\subfloat[$N = 24, \#active = 24$]{
\captionsetup{justification = centering}
\includegraphics[width=0.45\textwidth]{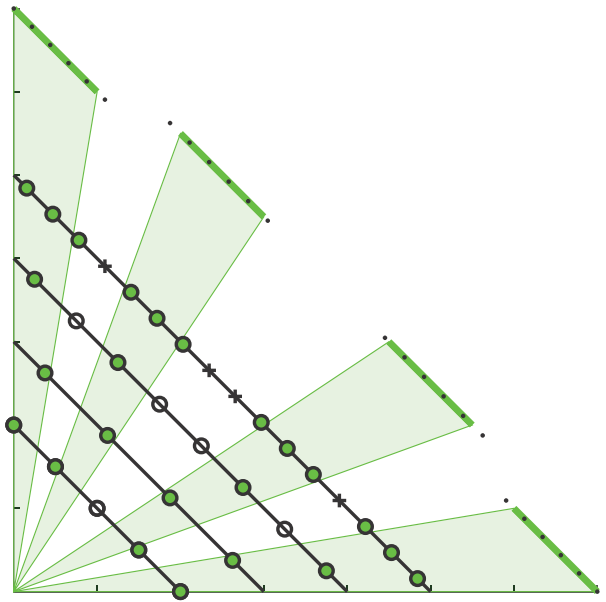}}
\caption{Results of adaptation on the simulated test cases: we simulate irregular and fractal PFs in $2D$ objective spaces to test the effectiveness of the adaptation. The dots on near the current PF shows the intersection points of the adapted reference vectors and the current PF. The $4$ plotted simulation results are persuasive: adaptation can deal with the complicated fractal PFs well, with maximum inaccuracy of $5/24 = 20.83\%$ in test case 3; But there is no need to worry, since such inaccuracy will decrease with the increment of $N$. The current $N = 24$ is too small, which is only chosen for visualization purposes.}
\label{fig:markov}
\end{figure*}
\par
Moreover, to test the Markov properties (the adaptation is expected to be Markovian since past adaptations should not influence the future), we shuffle the $4$ simulated test cases to $24$ possible permutations and output the similarity of the adaptation results (percentage of the identically enabled points in the RA) on these sequences in Tab. \ref{tab:markov}.
\begin{table*}[htbp]
\setlength{\tabcolsep}{3pt}
\scriptsize
\centering
\caption{Intra-sequence Similarity for Markov Property of Adaptation}
    \begin{tabular}{cccc}
    \toprule
    \toprule
    \multicolumn{4}{c}{\textbf{Similarity (\%)}} \\
    \midrule
    a     & b     & c     & d \\
    \midrule
    100\% & 100\% & 100\% & 100\% \\
    \bottomrule
    \bottomrule
    \multicolumn{4}{m{0.16\textwidth}}{\tiny Execution of adaptation on the permuted sequences is one by one, and continue upon the previous state of RA.}\\
    \end{tabular}%
  \label{tab:markov}%
\end{table*}%
\par
The test demonstrate perfect Markovian behavior for the adaptation mechanism, but these simulated test cases cannot represent general test cases. However, the results demonstrate some level of reliable stability of the adaptation engine.
\section{Conclusions}\label{section:conclusion}
Aiming to address the problems caused by infeasible parts of MaOPs, this paper proposed a reference vector based MaOEA with interacting components. The algorithm \algoabbr{} tracks the current PF and uses such information to adapt the reference vectors which guides the evolution. The novelty in this paper mostly lies in the designed interaction scheme among the components and the method of adapting reference vectors. The effectiveness of the work has been validated thoroughly in the comparative experimental studies. It can be concluded that the proposed algorithm \algoabbr{} is effective in dealing with the MaOPs and particularly good at solving the problems with partially feasible objective spaces.
\section*{Acknowledgments}
This work is supported by Mila (L'Institut qu\'eb\'ecois d'intelligence artificielle), the Scientific Computing Laboratory and Reasoning and Learning Laboratory, McGill University, the National Natural Science Foundation of China (61572104, 61103146, 61425002, 61751203), the Fundamental Research Funds for the Central Universities (DUT17JC04), and the Project of the Key Laboratory of Symbolic Computation and Knowledge Engineering of Ministry of Education, Jilin University (93K172017K03).

\section*{References}
\bibliography{mybibfile}
\end{document}